\pdfoutput=1
\documentclass{article}

\usepackage[dblblindworkshop,final]{neurips_2026}
\workshoptitle{Interpretability as a Science}

\usepackage[utf8]{inputenc}
\usepackage[T1]{fontenc}
\usepackage{hyperref}
\usepackage{url}
\usepackage{booktabs}
\usepackage{amsfonts}
\usepackage{nicefrac}
\usepackage{microtype}
\usepackage{xcolor}
\usepackage{graphicx}
\usepackage{caption}
\usepackage{enumitem}
\usepackage{amsmath}
\usepackage{amsthm}
\usepackage{multirow}
\usepackage{fontawesome5}
\usepackage{array}

\graphicspath{{figs/}}
\newcommand{\pit}{\textsc{pit}}
\newcommand{\tpit}{t_{\textsc{pit}}}

\newtheorem{definition}{Definition}
\newtheorem{proposition}{Proposition}

\newtheorem{lemma}{Lemma}
\title{Effectiveness of Hidden-State Probes\\for Localizing Reasoning Failures}

\author{%
  Shubham Subhnil\thanks{Work done during an internship at IBM Research.}\\
  Trinity College Dublin, Ireland\\
  \texttt{subhnils@tcd.ie}
  \And
  Yannis Belkhiter\thanks{Also a PhD student at Trinity College Dublin, Ireland.}\\
  IBM Research, Dublin, Ireland\\
  \texttt{belkhity@tcd.ie}
  \And
  Erik Miehling\\
  IBM Research, Dublin, Ireland\\
  \texttt{Erik.Miehling@ibm.com}
  \AND
  Adam Dahlgren Lindstr\"{o}m\\
  Ume\aa{} University, Sweden\\
  \texttt{dali@cs.umu.se}
  \And
  Elizabeth M. Daly\\
  IBM Research, Dublin, Ireland\\
  \texttt{elizabeth.daly@ie.ibm.com}
}

\begin{document}

\maketitle

\begin{abstract}
Verifying the final answer of a long reasoning trace may show that the trace is wrong but it does not distinguish three cases. The problem may lie beyond the model's capability, another attempt would succeed, or one step in an otherwise recoverable trace made the trace fail. Standard hidden-state probes are trained on final-answer correctness, a label that conflates these cases. Monte Carlo rollouts distinguish the cases and locate the failing step at additional inference cost. We test the hidden states of 4B and 8B reasoning models for the failure type and the failing step and compare each probe with the strongest baseline that uses no hidden states. On incorrect 8B traces, the confidence probe, a calibrated correctness probe at a midpoint cutoff, misses 81\% of cases in which a single step makes the trace unrecoverable and 47\% of cases in which another attempt would succeed. A probe with a different training design over the same states reverses the ordering of these two kinds of failure, so either ordering can be read from those states. At both scales, no single tested probe locates the failing step more accurately than a rule based only on trace length, and at 4B a probe trained to predict that step fails to outperform this rule in any of 10 runs on identical traces. The confidence probe therefore indicates that an answer is wrong, but its failure-type ordering depends on its training design.
\end{abstract}

\begin{center}
\small \faGithub\ \href{https://github.com/SSubhnil/Reasoning-Failures}{\texttt{github.com/SSubhnil/Reasoning-Failures}}\\
\raisebox{-0.25\height}{\includegraphics[height=1.35em]{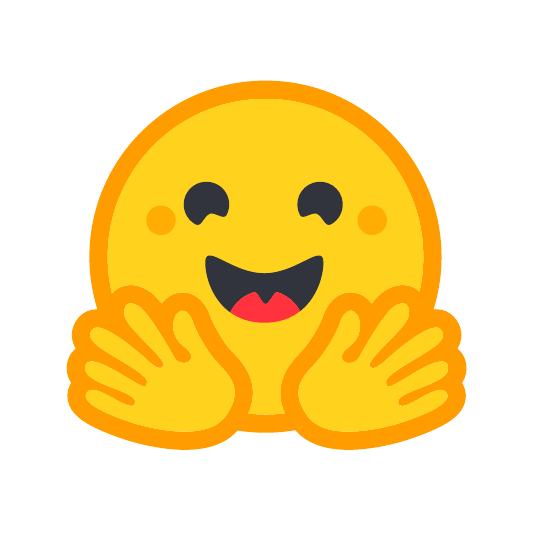}}\ \href{https://huggingface.co/datasets/ssubhnil/Reasoning-Failures}{\texttt{huggingface.co/datasets/ssubhnil/Reasoning-Failures}}
\end{center}

\section{Introduction}
\label{sec:intro}

Reasoning models can produce long traces that end in an incorrect answer \citep{guo2025deepseekr1,yang2025qwen3}. For these incorrect traces, the appropriate intervention depends on the cause of the failure. Some problems remain unsolved across many sampled attempts \citep{yue2025rlvrceiling}; others are solved on another attempt \citep{cotpassk2025}; and some traces remain recoverable until one step makes all later continuations fail \citep{hwang2024selfexplore}. A final label does not distinguish these failure types, and a model's stated confidence is an unreliable guide to correctness: instruction-tuned models asked for a confidence score often assign high scores to incorrect answers \citep{zhao2026overconfidence}. A monitor therefore needs to identify the failure type and, for a trace that becomes unrecoverable, the step at which recovery is lost.

A \emph{trace} is the reasoning the model generates in one attempt at a problem, and a \emph{prefix} is a trace up to a given step. A prefix is \emph{recoverable} if a correct answer remains reachable from it. Monte Carlo rollouts from each evaluated prefix estimate the probability that a continuation from it reaches a correct answer. An incorrect trace is termed a \emph{capability gap} if it is unrecoverable from the initial problem state, and a \emph{near miss} if it remains recoverable at every evaluated prefix. A \pit{} is an incorrect trace whose first unrecoverable prefix occurs partway through \citep[cf.][]{setlur2024rlincorrect}; the step that ends that prefix is its \emph{\pit{} step}. We test the ability of hidden-state probes to distinguish these failure types and locate the \pit{} step without additional rollouts at deployment time.

Hidden states carry information about the truth of a statement or an answer \citep{azaria2023saplma,marks2023geometryoftruth,kadavath2022pik}, and a probe on the hidden states of a language model can estimate the correctness of its reasoning steps \citep{reprobe2026}. Final labels record only final-answer correctness and do not directly supervise the failure type or the \pit{} step. We train probes on the hidden states at a fixed set of layers, the \emph{read site}, with three kinds of supervision: final labels; \emph{rollout rates}, the fraction of rollouts from a prefix that reach a correct answer; and the \pit{} step itself (Table~\ref{tab:probes}).

We evaluate Qwen3 in thinking mode at 4B and 8B on hard mathematics problems from DeepMath-103k \citep{he2025deepmath} and the training split of Hendrycks MATH \citep{hendrycks2021math}. We use this model family because it provides open-weight reasoning models at two scales with a thinking mode \citep{yang2025qwen3}. Appendix~\ref{app:proto} gives the difficulty levels, the sampling parameters, and the seed and interval conventions of each measurement.

\paragraph{Findings.}
\begin{enumerate}[leftmargin=*,itemsep=1pt,topsep=2pt]
\item \textbf{The confidence probe, trained on final labels, misses capability gaps and \pit{}s more often than near misses.} At the final step, the probe predicts the final label and remains calibrated at both model scales, but it misses $81\%$ of \pit{}s and $47\%$ of near misses at 8B; a trace is missed when its predicted correctness at the final step is at or above a \emph{cutoff} of $0.5$. After controls for four variables of the trace and the problem, the difference between \pit{}s and near misses exceeds the prespecified minimum effect of $0.08$ at 8B. At 4B the difference is nonzero and below the minimum effect (Section~\ref{sec:blind}).
\item \textbf{The combined probe reverses the failure-type ordering.} The combined probe, trained on rollout rates of incorrect traces, reads the same hidden states at the same read site. At 8B the confidence probe gives a randomly chosen \pit{} a higher final-step score than a randomly chosen near miss with probability $0.76$ (the AUROC of \pit{}s against near misses), and the combined probe does so with probability $0.40$. The two probes differ in training labels, attention mask, step weighting and training schedule, so the reversal cannot be attributed to the labels alone (Section~\ref{sec:blind}).
\item \textbf{At 8B, a failure-type classifier on transcoder activations meets the margin.} The failure-type classifier reads the transcoder activations of one completed incorrect trace per problem and separates problems labelled capability gap from problems labelled \pit{}. It exceeds the strongest tested baseline that uses no hidden states by $0.055$ AUROC at 8B and $0.049$ at 4B, so the 8B result meets the prespecified margin of $0.05$ and the 4B result does not (Section~\ref{sec:blind}).
\item \textbf{No single probe exceeds the length rule at locating the \pit{} step.} The length rule predicts a late step from trace length alone. At 4B, the pointer probe, trained to predict the \pit{} step, exceeds the length rule in 0 of 10 runs on identical held-out \pit{}s. The ensemble, the average of the five pointer probes, exceeds the length rule under one of the two weightings of its failure-type loss. At 8B only the value probe, trained on rollout rates of recoverable traces, was tested, and it is below the length rule in each of five training seeds (Section~\ref{sec:wall}).
\end{enumerate}

The localization finding is limited to the probe designs, training labels, read site, and the two collections of rollouts specified in Table~\ref{tab:probes} and Section~\ref{sec:failure}.

\section{Failure types}
\label{sec:failure}

We describe the two collections of rollouts that supply the labels, define the continuation value of a prefix and the best reachable value, use a threshold $\theta$ on the best reachable value to assign each incorrect trace a failure type and each \pit{} its \pit{} step, and state the evaluation protocol.

\paragraph{Collections of rollouts.}  Two collections of rollouts supply the labels. The \emph{broad collection} labels every incorrect trace at both scales. It evaluates steps 0, 1 and 2 and then every second step of a trace with 8 rollouts, and evaluates some of these prefixes with 16 rollouts in a refinement stage (Appendix~\ref{app:proto}). The \emph{gold collection} applies the same scan to a set of 4B traces with 4 rollouts and refines with 8, half as many as the broad collection. The \pit{} steps of its 818 \pit{}s are the \emph{reference \pit{} steps} for localization. Each comparison is scored on one collection, because the two collections use different numbers of rollouts. Table~\ref{tab:sets} gives the size of each population and the comparisons that use each population.

\begin{figure}[t]
\centering
\includegraphics[width=\textwidth]{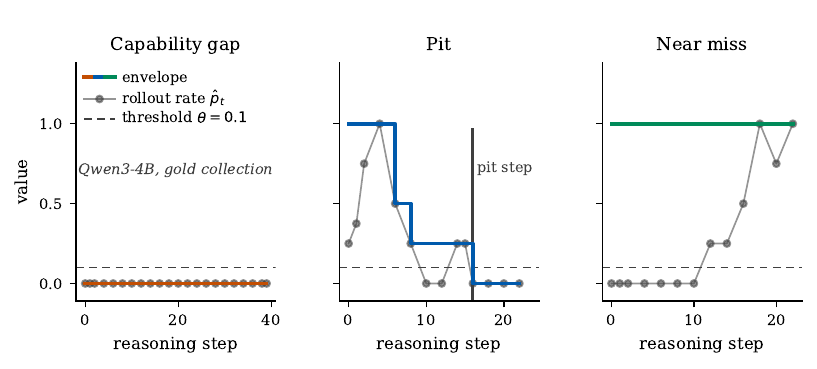}
\caption{The three failure types (Definition~\ref{def:pit}) on three Qwen3-4B traces from the gold collection. Each panel shows, for one trace, the rollout rates $\hat p_t$ at the evaluated prefixes and the envelope, which is the best reachable value $V^*_t$ at each step. For a capability gap, the envelope is at or below $\theta=0.1$ from step 0. For a \pit{}, the envelope first falls to $\theta$ or below partway through the trace, at the \pit{} step. For a near miss, the envelope stays above $\theta$ at every step.}
\label{fig:reach}
\end{figure}

\paragraph{Step-level value.}
We split each trace at blank lines and merge adjacent blocks until every step contains at least 400 tokens. For problem $\mathbf{x}$, model $\pi$ produces steps $\hat y_1,\dots,\hat y_T$, and the prefix $s_t=(\mathbf{x},\hat y_{1:t})$ contains the problem and the steps through $t$. The final label $Y\in\{0,1\}$ records the correctness of the answer generated after the reasoning. Following prior work on mathematical reasoning \citep{setlur2024rlincorrect,yu2024ovm,kazemnejad2024vineppo,wang2024mathshepherd}, we define the \emph{continuation value} $V^\pi_t$ of a prefix as the probability that a continuation sampled from it ends with a correct answer. We estimate it by the rollout rate $\hat p_t$, the fraction of $K$ rollouts from $s_t$ that end with a correct answer:
\begin{equation}
\label{eq:value}
V^\pi_t=\mathbb{E}_{\pi}\big[\,Y \mid s_t\,\big],
\qquad
\hat p_t=\frac{1}{K}\sum_{j=1}^{K} Y^{(j)} .
\end{equation}
We call predicting the final label \emph{detection}, estimating $V^\pi_t$ \emph{value estimation}, and locating the \pit{} step \emph{localization}.

\paragraph{Pit-step definition.} Process reward models trained on human step labels target local validity, the correctness of each step judged on its own \citep{lightman2023verify,uesato2022processoutcome}. Correct traces can contain locally flawed steps \citep{zheng2024processbench}, and in long traces that revise earlier work a local checker can flag a different step from the \pit{} step. We therefore define the \pit{} step from the continuation value in Equation~\ref{eq:value}.

\begin{definition}[Best reachable value and the \pit{} step]
\label{def:pit}
For an incorrect trace, the \emph{best reachable value} at step $t$ is $V^*_t = \max_{t'\geq t} V^\pi_{t'}$, which is non-increasing in $t$, and the \emph{envelope} of the trace is the sequence $V^*_0,\dots,V^*_T$. Step $0$ denotes the problem before any reasoning, with $s_0=(\mathbf{x})$. Let $\tpit$ be the first step at which the envelope is at or below the threshold $\theta$:
\begin{equation}
\tpit = \min\{\, t \in \{0,1,\dots,T\} : V^*_t \leq \theta \,\}, \qquad \theta = 0.1.
\end{equation}
An incorrect trace is a capability gap if $\tpit=0$, a \pit{} if $1\le\tpit\le T$, in which case $\tpit$ is its \pit{} step, and a near miss if the envelope stays above $\theta$ at every step, written $\tpit=\infty$. Failure types are assigned to incorrect traces only. The running maximum $V^*$ is defined along one sampled trace and differs from the optimal value of a policy.
\end{definition}

\noindent In the collections, the envelope is computed from the rollout rates $\hat p_t$ at the evaluated prefixes, with the maximum taken over the evaluated prefixes at or after step $t$, so the \pit{} step is always the last step of an evaluated prefix. $V^*_t$ gives the highest continuation value available at the current or any later step in the sampled trace. For example, a trace can establish a valid approach, introduce a sign error, and then develop the resulting wrong branch. Its envelope is high before the sign error and near zero afterward, so its \pit{} step is the step that introduces the sign error.

The threshold $\theta=0.1$ exceeds $1/16=0.0625$, the rollout rate of a single success among 16 rollouts, the largest $K$ used to assign labels. At a prefix where the envelope is at or below $\theta$, the rollout rate there and at every later evaluated prefix therefore reflects at most one success in 16 rollouts and no success in 8 or fewer. Self-Explore treats a prefix as unrecoverable only if no sampled continuation succeeds \citep{hwang2024selfexplore}; the threshold relaxes this test so that a single success in 16 rollouts does not make a prefix recoverable.

\begin{figure}[t]
\centering
\includegraphics[width=\textwidth]{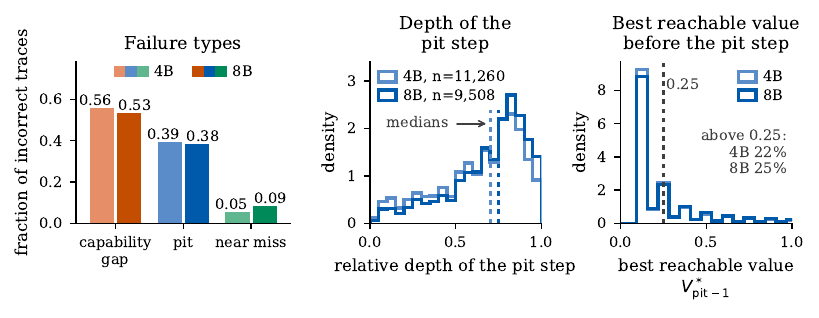}
\caption{Failure types and \pit{} steps in the broad collection at both model scales (Table~\ref{tab:sets}). \textbf{Left:} Shares of capability gaps, \pit{}s and near misses among incorrect traces. \textbf{Center:} The relative depth of the \pit{} step, the \pit{} step divided by the number of steps in the trace, with the medians marked. The late median depths of the \pit{} step motivate a baseline that predicts a late step. \textbf{Right:} The best reachable value immediately before the \pit{} step, $V^*_{\tpit-1}$, with the share above $0.25$ marked. $V^*_{\tpit-1}$ depends on later sampled states, so a probe that reads only the steps generated so far need not determine it (Section~\ref{sec:wall}).}
\label{fig:anat}
\end{figure}
\citet{setlur2024rlincorrect} select the first step with the lowest estimated value in an incorrect trace, $\arg\min_t \hat V_t$. With few rollouts per prefix, zero estimates are common and a zero estimate can be followed by a rate above $\theta$, so the $\arg\min$ step and the \pit{} step rarely coincide at either scale (Table~\ref{tab:support}); all \pit{} steps in this paper follow Definition~\ref{def:pit}.

A capability gap has no recoverable prefix; it is related to problems that remain unsolved across repeated attempts \citep{yue2025rlvrceiling} and to difficulty regions in which accuracy is at most 10 percent and does not rise with more sampled reasoning paths \citep{reasoningboundary2024}. A near miss remains recoverable at every evaluated prefix although the trace ends with a wrong answer; it is related to problems that are solved on another attempt \citep{cotpassk2025}. A \pit{} is recoverable before its \pit{} step and unrecoverable from it onward (Figure~\ref{fig:reach}). Near misses and \pit{}s are the incorrect traces whose prefix at step 0 is recoverable; we call them \emph{recoverable traces}.

Capability gaps are the most frequent failure type at both scales, followed by \pit{}s and near misses, and fewer than half of the incorrect traces are recoverable (Figure~\ref{fig:anat}). The shares of the failure types change by less than two percentage points for $\theta$ between $0.05$ and $0.10$, and raising $\theta$ to $0.15$ moves some \pit{}s to capability gaps (Table~\ref{tab:claims}). The shares depend weakly on problem type (Appendix~\ref{app:types}). The best reachable value immediately before the \pit{} step is usually low but has a right-skewed tail: some \pit{}s retain a best reachable value above $0.25$ until their \pit{} step (Figure~\ref{fig:anat}). \pit{} steps occur late in the trace, so localization is compared with the length rule, which predicts a late step (Section~\ref{sec:wall}).

The behavioral taxonomies reviewed in Section~\ref{sec:related} group near misses and \pit{}s together, as failures that resampling fixes. Capability gaps, near misses and \pit{}s are the evaluation targets of the remaining experiments.

\paragraph{Evaluation protocol.} The two scales share the problem set, but the traces, and therefore the incorrect traces and their failure types, differ between scales, so no comparison between scales is paired. Each point estimate has a 95\% bootstrap interval obtained by resampling problems. We report no detected difference when the interval contains zero. Two decision thresholds were fixed before the probes were evaluated. First, a probe counts as using information in the hidden states only if its AUROC exceeds that of the strongest tested baseline for its task that uses no hidden states by $0.05$, the \emph{margin}, with an interval of the difference that excludes zero. Second, the difference between \pit{}s and near misses after controls in Section~\ref{sec:blind} counts as an effect only above $0.08$, the \emph{minimum effect}.

\section{Confidence predictions by failure type}
\label{sec:blind}

A probe trained on hidden states and final labels needs no rollouts at deployment time. We train the confidence probe, validate its detection AUROC and calibration, and then evaluate its predictions at the final step.

\paragraph{Probe inputs.} The probes read the hidden states of the frozen model at the read site, on traces that are generated once and reused for every probe (Table~\ref{tab:probes}). For each step, a probe reads the hidden states at the read site at one token position, which lies at or before the last token of the step. For about half of the traces, the position read for the final step lies in an earlier step (Appendix~\ref{app:repr}). We selected the read site, layers $\{23,24\}$ at both scales, from the detection AUROC of a probe at each layer, using the final label only and no failure type or \pit{} step (Table~\ref{tab:support}; Appendix~\ref{app:repr}).

\paragraph{Confidence probe.} We call the probe trained on final labels the \emph{confidence probe}. Every confidence score in this paper is an output of the confidence probe. Our confidence probe is a causal transformer encoder over the steps of a trace. At step $t$ it reads the hidden states $h_1,\dots,h_t$ at the read site and outputs the predicted correctness
$c_t = \sigma(g(h_{1:t})) = \widehat{\Pr}(Y{=}1 \mid h_{1:t})$. We broadcast the trace's final label to every valid step and minimize per-step Bernoulli loss. The same final label supervises every prefix, so the objective estimates $\mathbb{E}[Y\mid h_{1:t}]$ \citep{yu2024ovm}, which equals $V^\pi_t$ only when $h_{1:t}$ retains all the information in the prefix about the final label (Proposition~\ref{prop:detect}, Section~\ref{sec:wall}). We calibrate held-out scores by temperature scaling \citep{guo2017calibration}. Training and validation traces are split by problem, so no problem contributes traces to both; we call this the \emph{problem split}. Appendix~\ref{app:proto} gives the architecture and training details and reports a split by trace for comparison.

On held-out problems, the confidence probe predicts the final label and is calibrated before temperature scaling at both scales (Table~\ref{tab:conf}). Trace length, answer form and problem difficulty also predict the final label, so a probe can score well without using information beyond these variables; format variables have been shown to account for probe accuracy on a different prediction target \citep{sahoo2026taskformat}. Table~\ref{tab:conf} therefore also reports the detection AUROC of the probe's score after five such variables are regressed out, and a variant of the probe trained with the Brier loss, which gives similar detection AUROC and calibration.

\paragraph{Text baseline.} Among the baselines that use no hidden states, the strongest we tested for detection is the \emph{text baseline}, a gradient-boosted classifier that predicts the final label from eleven features of the trace text, such as the frequency of verification and backtracking sentences, the position of the final answer, and the number and length of sentences (Appendix~\ref{app:proto}). A logistic regression on the same features has a lower detection AUROC, so a linear classifier understates the predictive power of these features (Table~\ref{tab:conf}). The confidence probe reaches a detection AUROC of $0.81$ at both scales, against $0.78$ at 4B and $0.77$ at 8B for the text baseline. The interval of the difference excludes zero and lies below the margin at both scales, so the confidence probe exceeds the text baseline without meeting the margin. Adding the confidence probe's score to the text baseline's features raises the detection AUROC, with an interval that excludes zero (Table~\ref{tab:conf}).

\paragraph{Failure-type pattern.} The confidence probe misses capability gaps and \pit{}s more often than near misses at both scales (Figure~\ref{fig:conceal}A; Table~\ref{tab:conceal}); a trace is missed when its predicted correctness at the final step is at least the cutoff of $0.5$. Overall calibration is compatible with this pattern by failure type, because incorrect traces are the minority of the calibration population.

\paragraph{Controlled comparison.} We test the hypothesis that the best reachable value at step 0, $V^*_0$, trace length, answer position and problem difficulty account for the higher miss rate of \pit{}s than of near misses. We first compare the two failure types within bins of $V^*_0$ (Figure~\ref{fig:conceal}B), and then fit a logistic model of trace misses on the failure type and the four variables. The \emph{difference between \pit{}s and near misses after controls} is the average marginal effect of changing a trace's failure type from near miss to \pit{} in this model (Appendix~\ref{app:proto}). It exceeds the minimum effect at 8B, at $+0.125\,[+0.104,+0.143]$. At 4B the difference is nonzero and below the minimum effect, at $+0.037\,[+0.012,+0.064]$ (Table~\ref{tab:conceal}). Capability gaps have $V^*_0\le\theta$ by definition and \pit{}s have $V^*_0>\theta$, so the two failure types occupy disjoint ranges of $V^*_0$, and their miss rates are compared without controls.

\paragraph{Combined probe.} We call the probe trained on rollout rates of incorrect traces the \emph{combined probe}. It reads the confidence probe's input, the hidden states at the read site, and keeps the final label as the target on correct traces. At the final step, the confidence probe gives \pit{}s higher scores than near misses at both scales, and the combined probe gives \pit{}s lower scores than near misses (Figure~\ref{fig:conceal}C); each probe is scored on its own held-out traces. The two probes differ in training labels, attention mask, step weighting and training schedule (Table~\ref{tab:probes}; Appendix~\ref{app:proto}), so this reversal cannot be attributed to the labels alone.

\paragraph{Failure-type classifier.} The confidence probe's miss rates for capability gaps and \pit{}s differ by less than two percentage points at both scales (Table~\ref{tab:conceal}), so we test the separation of the two failure types from a completed incorrect trace. A transcoder is a fixed sparse approximation of an MLP sublayer \citep{dunefsky2024transcoders,paulo2025transcoders}. The \emph{failure-type classifier} differs from the confidence probe in input and in scoring problems instead of traces: it reads the transcoder activations at layers 19 to 24 of one completed incorrect trace per problem and separates problems labelled capability gap from problems labelled \pit{} (Appendix~\ref{app:presence}). It exceeds the strongest tested baseline that uses no hidden states by $0.055$ AUROC at 8B and $0.049$ at 4B. The margin is applied to the mean over six repetitions of five-fold cross-validation on the same problems, so the 8B result meets the margin and the 4B result does not. At both scales the interval of the difference excludes zero and contains the margin in each of the six repetitions. The same classifier on the hidden states at these layers scores lower at both scales (Appendix~\ref{app:presence}). Our intuition is that the width of the transcoder input, 40 to 64 times that of the hidden states, accounts for part of this gap; a random encoding of equal width was not tested. The failure-type classifier reads a completed trace and predicts a label of its problem, so it applies only after a trace has ended incorrectly. Appendix~\ref{app:presence} reports the baselines, the six repetitions, and the differences after removing problems whose labels may be defective.

\paragraph{Predicted correctness at the \pit{} step.} At both scales, the confidence probe's predicted correctness is lower at the \pit{} step than at the other steps of the trace. The \emph{position-matched steps} of a trace are its other steps whose scaled position $t/(L-1)$, for step $t$ of a trace with $L$ steps, lies within $0.15$ of the \pit{} step's. Against these steps the decrease remains at 8B. At 4B, where the comparison uses about a tenth as many \pit{}s, there is no detected difference (Table~\ref{tab:support}).

\begin{figure}[t]
\centering
\includegraphics[width=\textwidth]{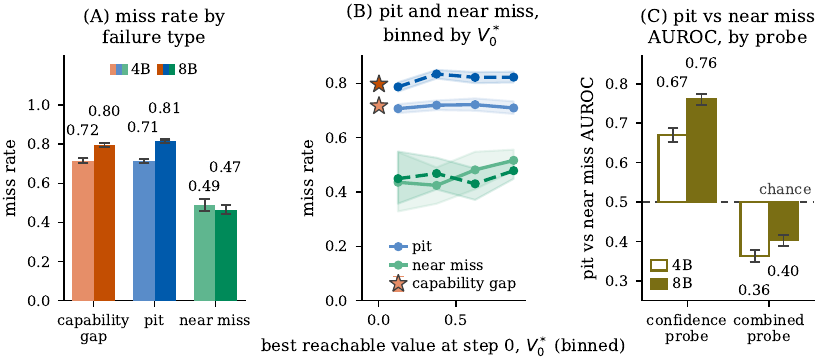}
\caption{Confidence-probe predictions at the final step, on the broad collection at both model scales, with training and validation split by problem. \textbf{(A)} The miss rates for capability gaps, \pit{}s, and near misses. \textbf{(B)} The miss rates of \pit{}s and near misses within four bins of the best reachable value at step 0, $V^*_0$, at 4B (solid) and 8B (dashed); the star marks the miss rate of capability gaps at their mean $V^*_0$. The incorrect traces differ between scales, so no comparison between scales is paired. Bands and bars are 95\% problem-clustered bootstrap intervals at both scales. \textbf{(C)} AUROC of \pit{}s against near misses: the confidence probe ranks \pit{}s above near misses and the combined probe ranks \pit{}s below near misses, each by its final-step score on its own held-out traces. Table~\ref{tab:conceal} gives the values, intervals and trace counts.}
\label{fig:conceal}
\end{figure}

\section{Locating the \pit{} step}
\label{sec:wall}

We compare probes that locate the \pit{} step with the strongest tested baselines that use no hidden states.

\paragraph{Identifiability.} A training signal identifies a quantity when the population optimum of its objective determines that quantity. The \pit{} step is defined through $V^*_t$, a maximum over current and later prefix values. At generation time a probe observes only the prefix, so the later values can remain random conditional on its input.

\begin{proposition}[Population optimum of outcome supervision]
\label{prop:detect}
Let $N$ be the number of valid steps in a trace and assume $0<\mathbb E_\pi[N]<\infty$. Under a strictly proper per-step Bernoulli objective with final label $Y$ broadcast to every valid step, the population risk minimizer over measurable $[0,1]$-valued functions of the hidden prefix is $g^\star(h_{1:t})=\mathbb{E}_{\mathbb{Q}}[Y\mid h_{1:t}]$. Here $\mathbb{Q}$ is the step-weighted training law, which counts each trace once per valid step. Moreover, $g^\star(h_{1:t})=\mathbb{E}_{\mathbb{Q}}[V^\pi_t\mid h_{1:t}]$, and it equals $V^\pi_t$ under representation sufficiency.
\end{proposition}

\begin{proposition}[Exact recovery from a prefix]
\label{thm:nonident}
Under the completed-incorrect-trace law, fix step $t$, condition on $T\ge t$, and let $s_t$ denote the prefix. A measurable function $f$ with $f(s_t)=V^*_t$ almost surely exists if and only if $\mathrm{Var}(V^*_t\mid s_t)=0$ almost surely. No function of $s_t$ returns the realized $V^*_t$ exactly where that variance is positive, regardless of training label or loss. $\tpit$ is a function of $s_t$ if and only if its conditional law given $s_t$ is a point mass. For $m_t=V^*_t-V^\pi_t$, $\mathrm{Var}(m_t\mid s_t)=\mathrm{Var}(V^*_t\mid s_t)$.
\end{proposition}
\noindent A hidden-state probe is no stronger than a probe of $s_t$, since $h_{1:t}$ is a deterministic function of $s_t$. The claim concerns exact recovery of the realized quantity; a probe may still estimate its conditional mean. Proposition~\ref{thm:nonident} applies only where $\mathrm{Var}(V^*_t\mid s_t)>0$. We do not measure this variance, and $0<V^\pi_t<1$ does not determine it, so Proposition~\ref{thm:nonident} does not explain the localization results below. The values along a completed trace determine $\tpit$ (Lemma~\ref{lem:posthoc}); the probes below estimate those values or predict the \pit{} step directly. Both propositions are proved in Appendix~\ref{app:proofs}.

\paragraph{Length rule.} The \emph{length rule} predicts the \pit{} step $t=\mathrm{round}(f\,(L-1))$ of a trace with $L$ steps, with one parameter $f$. \pit{} steps occur late in a trace: the median relative depth of the \pit{} step is $0.70$ at 4B and $0.75$ at 8B (Figure~\ref{fig:anat}). We score the probes and the length rule by Hit@$\pm2$, the fraction of \pit{}s for which the predicted step lies within two steps of the \pit{} step. In each comparison, $f$ is fitted to maximize Hit@$\pm2$ on \pit{}s of problems outside the evaluation set. The length rule is therefore refitted for each evaluation set, because its score depends on that set's trace lengths and on the relative depths of its \pit{} steps (Appendix~\ref{app:proto}).

\paragraph{Pointer probe.} We call the probe trained to predict the \pit{} step the \emph{pointer probe}. It differs from the confidence probe in input and target: it reads a completed trace, which Proposition~\ref{thm:nonident}, a statement about prefixes, does not cover, and is trained on the \pit{} step itself. The pointer probe is a transformer encoder with bidirectional attention over the steps of a completed trace. At each step it reads that step's hidden states at the read site together with those of the final step, and it outputs a distribution over the steps of the trace, whose most probable step is its prediction. The distribution is trained on the \pit{} steps of the broad collection's \pit{}s. It is trained jointly on three further outputs, among them the failure type, which is trained on all incorrect traces with one of two weightings of its loss, inverse-frequency or uniform (Appendix~\ref{app:proto}); only the step distribution is evaluated.

\paragraph{Paired localization results.} At 4B, the pointer probe is trained on the broad-collection traces of 80\% of the problems and evaluated on the $2{,}289$ \pit{}s of the remaining problems; the same problem split is used for every seed and both weightings of the failure-type loss. The pointer probe exceeds the length rule in 0 of 10 runs on identical held-out \pit{}s (Table~\ref{tab:paired}). Each comparison is paired on the same \pit{}s, with intervals that resample problems. In a second comparison, on the reference \pit{} steps of the gold collection, a pointer probe trained without any problem that contributes a gold trace does not exceed the length rule in any of five seeds. The length rule in this comparison is fitted on broad-collection \pit{}s that exclude the gold problems (Table~\ref{tab:paired}).

An untrained \emph{distance score}, the distance of each step's stored hidden state from the pooled step states of correct traces of the same problem, has an AUROC for the \pit{} step near $0.5$ at every layer tested on 76 \pit{} traces at 4B (Figure~\ref{fig:wall}, right). At the read site it is below $t/(L-1)$ used as a score (Table~\ref{tab:support}).

The pointer probe and the length rule disagree on about a third of the held-out \pit{}s in every run, about equally in each direction (Table~\ref{tab:support}). Averaging the pointer probes across seeds, separately for each weighting of the failure-type loss, gives the \emph{ensemble}. With inverse-frequency weights, the ensemble exceeds the length rule with an interval of the difference that excludes zero; with uniform weights, the interval of the difference contains zero (Table~\ref{tab:paired}).

\paragraph{Value probe.} We call the probe trained on rollout rates of recoverable traces the \emph{value probe}; it has the combined probe's design, restricted to recoverable traces (Appendix~\ref{app:proto}). Its predicted \pit{} step is the \pit{} step of Definition~\ref{def:pit} computed from its predicted values at the evaluated prefixes in place of the rollout rates. The value probe's prediction at the last evaluated prefix separates held-out near misses from \pit{}s with an AUROC of $0.91$ at both scales (Table~\ref{tab:support}). The comparison of near misses with \pit{}s was not run for a baseline that uses no hidden states. By Definition~\ref{def:pit}, a near miss has a rollout rate above $\theta$ at its last evaluated prefix and a \pit{} does not, so the value probe's training target itself separates the two failure types and this AUROC measures the probe's fit to that target. The value probe's per-step mean absolute error exceeds the error from predicting zero in every seed, so rollout supervision does not yield accurate per-step values in this experiment.

\paragraph{Localization at 8B.} At 8B, where no gold collection exists, the value probe is the only probe tested for localization. In each seed it is scored on the \pit{} steps of the broad collection for that seed's held-out problems, on the same \pit{}s as the length rule. The comparison is through the five seed means, without a paired interval, and the value probe is below the length rule in every seed with both input sets (Table~\ref{tab:support}).

\paragraph{Final-label comparison.} We compare the value probe with the \emph{final-label probe}, a transformer encoder over the steps trained on final labels with a causal attention mask, a value output and a confidence output; its predicted \pit{} step is the step with the largest drop in its value output from the preceding step (Appendix~\ref{app:proto}). The two probes differ in training label, attention mask, output head and step rule, so the comparison changes the label together with these design choices. Each probe is run with five seeds and two input sets, hidden states alone and hidden states with five next-token statistics, on the reference \pit{} steps. The value probe is trained on broad-collection traces, which include most gold \pit{} traces, so we evaluate each seed on the reference \pit{} steps of the problems held out from the value probe's training. The final-label probe's training traces are not recorded, so some of these \pit{}s may be among them. On these \pit{} steps, both probes are below the length rule in every seed with both input sets (Table~\ref{tab:paired}). The value probe's advantage over the final-label probe has an interval that excludes zero in one of the five seeds with hidden states alone and in none with hidden states with the next-token statistics (Table~\ref{tab:support}). The value probe therefore gives no consistent improvement over the final-label probe on these \pit{} steps, and neither probe exceeds the length rule on these \pit{} steps or in a pooled evaluation over all reference \pit{} steps (Appendix~\ref{app:proto}).

\begin{figure}[t]
\centering
\includegraphics[width=\textwidth]{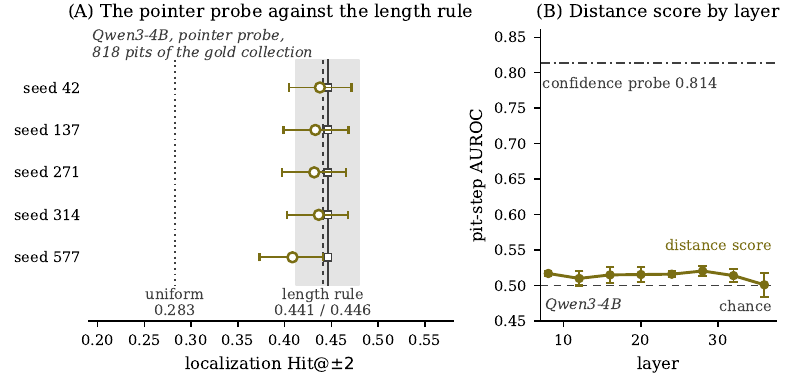}
\caption{Localization at 4B. \textbf{Left:} Hit@$\pm2$ on the reference \pit{} steps for the five pointer probes trained without any problem that contributes a gold trace. Vertical lines mark the uniform-step baseline (dotted, no interval) and two fits of the length rule: the fit on broad-collection \pit{}s outside the gold problems (solid, interval shaded), with which each pointer probe is paired, and the fit by five-fold cross-validation within the gold collection (dashed). Bars and the shaded band are 95\% bootstrap intervals. \textbf{Right:} AUROC of the distance score for the \pit{} step against the other steps of the trace, by layer, on 76 \pit{} traces. Points are midpoints and bars the range over four distance measures, not confidence intervals. The horizontal line is the detection AUROC of the confidence probe, a different task, shown for reference.}
\label{fig:wall}
\end{figure}

\paragraph{Rollout search.} The \pit{} step can also be located with rollouts instead of a probe \citep{luo2024omegaprm}. We replay such a search offline on the stored rollout rates of the gold collection: it reads every second step of a trace and then bisects the interval that contains the first prefix at which the envelope is at or below $\theta$. The search finds the reference \pit{} step within two steps for $80.4\%$ of the gold \pit{}s while reading $6.6$ evaluated prefixes per trace (Table~\ref{tab:support}). The search reads the rollout rates that define the reference \pit{} steps, so its agreement with them is a different quantity from a probe's Hit@$\pm2$ and is reported separately.

\section{Discussion and limitations}
\label{sec:discuss}

The confidence probe misses capability gaps and \pit{}s more often than near misses at both scales (Figure~\ref{fig:conceal}). At detection it exceeds the strongest baseline we tested that uses no hidden states, but it does not meet the margin (Section~\ref{sec:blind}). The text baseline's miss rates by failure type were not measured, so the pattern is established for the confidence probe only.

Either ordering of \pit{}s and near misses can be read from the hidden states that the confidence probe reads. The combined probe reads the same hidden states and gives \pit{}s lower final-step scores than near misses at both scales, the opposite of the confidence probe (Figure~\ref{fig:conceal}C). We think the training target itself produces the combined probe's order: by Definition~\ref{def:pit} a near miss keeps a rollout rate above $\theta$ at its last evaluated prefix and a \pit{} does not, so a probe fitted to these rates is trained toward ranking near misses above \pit{}s. The confidence probe and the combined probe differ in more than their labels, and we did not isolate the contribution of the labels (Section~\ref{sec:blind}).

At 4B, no single probe exceeds the length rule at locating the \pit{} step on held-out \pit{}s, and at 8B the one tested probe, the value probe, is below the length rule in every seed (Section~\ref{sec:wall}). The hidden states may still carry the \pit{} step; the tested probes do not recover it. Two limits bear directly on the localization result. First, each step's hidden states were read at a position at or before the step's last token, and for about half of the traces the position read for the final step lies in an earlier step (Section~\ref{sec:blind}). No result was re-scored with corrected positions, so the localization result holds for these stored hidden states. Second, the \pit{} steps are assigned from a small number of rollouts at each evaluated step, and the two collections assign different failure types to some traces they share (Appendix~\ref{app:proto}). The highest score that any predictor could reach against these labels is therefore unknown. We think the second limit dominates: the ensemble exceeds the length rule under one weighting, so the pointer probes read some information beyond trace length, and \pit{} steps assigned from 4 to 16 rollouts at every second step may be too noisy for that information to raise Hit@$\pm2$.

All results in this paper come from one model family, Qwen3 in thinking mode at two scales, on hard mathematics problems. The failure types depend on the rollout protocol, including the number of rollouts, the threshold and the token limits (Appendix~\ref{app:proto}). Under this protocol, a capability gap records only that no sampled continuation reached a correct answer. A rollout counts as successful when its final answer is correct, so a continuation that reaches the answer through an invalid argument counts as a success. Such cases are reported for base models \citep{cotpassk2025}. The final labels contain grading errors, and correcting them leaves capability gaps and \pit{}s missed more often than near misses at both scales (Appendix~\ref{app:proto}).

\section{Related work}
\label{sec:related}

\paragraph{Correctness probes.}
Probes on hidden states predict the truth of a statement or an answer \citep{azaria2023saplma,marks2023geometryoftruth,kadavath2022pik,orgad2024llmsknowmore}, and on reasoning models the correctness of an intermediate answer \citep{zhang2025reasoningknowright}. \citet{orgad2024llmsknowmore} also define error types by resampling the answers to a question and show that probes predict them; their types are defined per question from sampled final answers, and ours per trace from rollouts at each step. A related study trains correctness probes on instruction-tuned models and one distilled reasoning model \citep{diagnosticnotcausal2026}; the signal is decodable in all of these models. Steering along the probe direction in two instruction-tuned models does not raise accuracy consistently, which the authors read as a diagnostic signal. The internal signal is useful because verbal reports can differ from the model's computation \citep{turpin2023unfaithful,zhao2026overconfidence}. This paper measures the failure types that a correctness probe distinguishes.

\paragraph{Failure taxonomies.}
Uncertainty work commonly separates epistemic from aleatoric uncertainty \citep{ahdritz2024epistemic,abbasi2024believe} or data from model uncertainty \citep{uabench2026}, and a recent account separates input ambiguity, knowledge gaps and decoding randomness \citep{anatomyuncertainty2026}; decoding randomness is closest in definition to our near misses. Behavioral work separates failures that resampling fixes from failures it does not \citep{islah2026fixable}, and a similar division is hypothesized for self-correction \citep{errordepth2026}. Under our taxonomy, the failures that resampling fixes combine near misses with \pit{}s. Concurrent work predicts, from the output log-probabilities of failed rollouts, which intervention rescues a problem \citep{islah2026fixable}.

\paragraph{Pit-step targets.}
Process benchmarks label the first locally invalid step \citep{zheng2024processbench} or every erroneous section of a long trace \citep{he2025deltabench}, which differs from the \pit{} step, and annotators can disagree about which step is wrong \citep{jacovi2024reveal}. State geometry can identify a first local error \citep{reasoningbreak2026}, and attribution graphs can recover local validity \citep{zhao2025crv}. Both targets differ from the \pit{} step, which depends on the model's continuations from each prefix, so an oracle for local validity need not locate it. \citet{cui2025prime} and \citet{yuan2024implicitprm} infer step credit from final outcomes, which is a different target from the best reachable value (Proposition~\ref{prop:detect}). \citet{pointofnoreturn2026} locate commitment junctures by counterfactual resampling in strategic-deception settings. Among the definitions reviewed here, that counterfactual procedure is closest to the definition used in this study. Appendix~\ref{app:claims} compares each claim with the nearest prior result.

\section{Conclusion}

We assigned each incorrect trace of a reasoning model one of three failure types from rollouts and tested the recovery of these types and of the \pit{} step from the model's hidden states. The confidence probe predicts the final label, and at the cutoff it misses capability gaps and \pit{}s more often than near misses. The combined probe ranks \pit{}s and near misses in the opposite order to the confidence probe, so an ordering of failure types read from a probe depends on the probe's training design. The failure-type classifier, which reads transcoder activations of a completed incorrect trace, separates capability gaps from \pit{}s. It meets the margin at 8B and not at 4B. At 4B no single probe exceeds the length rule at locating the \pit{} step, and the ensemble exceeds it under one of the two weightings of its failure-type loss. At 8B only the value probe was tested, and it is below the length rule in every seed. The localization results concern the tested probes and read site; a probe that locates the \pit{} step from the hidden states remains possible.

\begin{ack}
This work was funded in part by the EU Horizon project ELIAS (\#101120237). Views and opinions expressed are those of the author(s) only and do not necessarily reflect those of the European Union or The European Research Executive Agency. This publication has partly emanated from research conducted with the financial support of Taighde \'Eireann -- Research Ireland under Frontiers for the Future grant number 21/FFP-A/8957 and grant number 18/CRT/6223. For the purpose of Open Access, the author has applied a CC BY public copyright licence to any Author Accepted Manuscript version arising from this submission. This work was partially supported by the Knut and Alice Wallenberg Foundation. Experiments were run on the Cognitive Compute Cluster (CCC) at IBM Research.
\end{ack}

\bibliographystyle{plainnat}
\bibliography{ref}

\newpage
\appendix

\section{Claims, evidence, and prior work}
\label{app:claims}
Table~\ref{tab:claims} states each claim at the scope supported by the experiments, reports the decisive measurements, records failed and inconclusive comparisons, and identifies the nearest prior work. The detailed tables provide sample sizes, seed axes, and confidence intervals. A probe counts as using information in the hidden states only when it meets the margin defined in Section~\ref{sec:failure}. Differences below the margin are recorded as failures even when their intervals exclude zero.
\begin{table}[htbp]
\scriptsize
\centering
\setlength{\tabcolsep}{3pt}
\begin{tabular}{>{\raggedright\arraybackslash}p{2.5cm}>{\raggedright\arraybackslash}p{3.5cm}>{\raggedright\arraybackslash}p{2.4cm}>{\raggedright\arraybackslash}p{2.0cm}>{\raggedright\arraybackslash}p{2.6cm}}
\toprule
\textbf{Claim} & \textbf{Supporting evidence} & \textbf{Verdict} & \textbf{Nearest prior work} &
\textbf{Contribution} \\
\midrule
Best reachable value partitions incorrect traces into capability gaps, \pit{}s, and near misses &
$0.56/0.39/0.05$ (4B) and $0.53/0.38/0.09$ (8B); recoverable $44.5$ and $46.8\%$; the shares change by less than two percentage points for
$\theta$ between $0.05$ and $0.10$, while $\theta{=}0.15$ raises the share of capability gaps to $0.68$ and $0.65$ &
Descriptive result under one decoding setting &
\citep{hwang2024selfexplore,setlur2024rlincorrect,errordepth2026,ahdritz2024epistemic} &
Separates \pit{}s from near misses among recoverable traces \\
The confidence probe predicts final correctness and remains calibrated &
Detection AUROC $0.814$ and $0.809$ over 5 seeds; after controls $0.729$ and $0.722$; ECE $0.014$/$0.024$
(Table~\ref{tab:conf}); the text baseline reaches $0.7783$ at 4B and $0.7725$ at 8B &
Supported for prediction and calibration; the probe exceeds the text baseline by less than the margin &
\citep{azaria2023saplma,kadavath2022pik,yuan2024implicitprm,diagnosticnotcausal2026} &
A step-level probe with a causal attention mask and surface-feature controls \\
Miss rates at the cutoff differ by failure type &
Miss rates $0.717/0.714/0.490$ (4B) and $0.797/0.815/0.466$ (8B); difference between \pit{}s and near misses after controls $+0.037$ and
$+0.125$ (Table~\ref{tab:conceal}) &
Supported at 8B; \textbf{partial at 4B}, where the difference between \pit{}s and near misses after controls is nonzero, its interval $[+0.012,+0.064]$ does not reach the minimum effect, and the confidence probe trained with the Brier loss gives $+0.052$ &
\citep{diagnosticnotcausal2026} &
Breaks down misses by failure type \\
Probes with different training designs recover opposite failure-type orderings &
AUROC of \pit{}s against near misses: confidence probe $0.67/0.76$, combined probe $0.36/0.40$ (4B/8B), each on its own held-out traces &
Opposite orderings under partially matched training designs; the contribution of the target alone is open &
\citep{yu2024ovm,lu2024autopsv} &
Compares two training designs on the same hidden states at the same read site \\
A failure-type classifier on transcoder activations of one completed incorrect trace per problem separates capability-gap problems from \pit{} problems &
Difference from the strongest tested baseline that uses no hidden states $+0.055$ (8B) and $+0.049$ (4B), mean over six repetitions of five-fold cross-validation; every interval excludes zero and contains the margin; the same classifier on hidden states scores lower (Appendix~\ref{app:presence}) &
Meets the margin at 8B, not at 4B; a difference between the scales is not established; applies after a trace has ended &
\citep{orgad2024llmsknowmore,stopbeforefail2026,kantamneni2025sparseprobing} &
Separates capability gaps from \pit{}s; the confidence probe's miss rates do not separate these two types \\
No single tested probe exceeds the length rule at locating the \pit{} step; at 8B only the value probe was tested &
Pointer probe $0.463$ versus length rule $0.469$, above the length rule in $0$ of $10$ runs; pointer probe trained without gold problems $0.429$ versus $0.446$, $0$ of $5$; ensemble $+0.026$ $[+0.002,+0.050]$ with inverse-frequency weights on the failure-type loss, interval containing zero with uniform weights; final-label comparison: the value probe and the final-label probe are below the length rule in every seed on the problems held out from the value probe's training; pooled over all reference \pit{} steps, which include training traces, the two probes score above the length rule in $0$ of $20$ runs (Table~\ref{tab:paired}); 8B value probe with the two input sets $0.429$ and $0.405$ versus length rule $0.503$, above the length rule in $0$ of $5$ seeds for each input set (Table~\ref{tab:support}) &
Negative result at 4B, paired on identical \pit{}s, and at 8B, compared through seed means against the \pit{} steps of the broad collection. Earlier configurations evaluated on $55$ to $134$ \pit{}s are inconclusive &
\citep{zheng2024processbench,reasoningbreak2026,zhao2025crv,cui2025prime,nadaf2026steerable} &
Uses a continuation-based \pit{}-step target and refits the length rule for each evaluation set \\
At 4B, a rollout search replayed on the gold collection's stored estimates reproduces the reference \pit{} steps &
Within two steps for $80.4\%$ of the gold \pit{}s while reading $6.6$ evaluated prefixes per trace; refinement-depth agreement $0.754$ &
Supported at 4B in an offline replay against the estimates that define the reference \pit{} steps; not comparable with a probe's Hit@$\pm2$; no interval &
\citep{luo2024omegaprm} &
Measures the share of the collection's evaluated prefixes that the search reads \\
\bottomrule
\end{tabular}
\caption{Claims, supporting evidence, verdicts, nearest prior work, and contributions. Among the reviewed studies, none combines a three-way recovery taxonomy, probes on hidden states, and a continuation-based \pit{}-step target.}
\label{tab:claims}
\end{table}

\section{Experimental protocols}
\label{app:proto}

\paragraph{Problem set and sampling.}
The two scales share one problem set of $35{,}854$ problems. Of these problems, $32{,}031$ come from DeepMath-103k at difficulty at least 7, and $3{,}823$ come from levels 4 and 5 of the training split of Hendrycks MATH. Each problem has four sampled traces, except 146 problems with eight, which gives $144{,}000$ traces at each scale. The traces were sampled at temperature $0.7$ with a 16k-token budget.

No problem in the problem set matches a MATH-500, MATH-Perturb or AIME 2024--2025 problem after whitespace and case normalization (0 of $35{,}854$ problems). The exact-match check does not detect paraphrased or reformatted duplicates.

\paragraph{Monte Carlo collections.}
The two collections use different numbers of rollouts, and every result identifies its source collection (Table~\ref{tab:collections}). The broad collection labels all incorrect traces and supplies the failure-type shares in Section~\ref{sec:failure}. In a scan, it evaluates steps 0, 1 and 2 and then every second step with 8 rollouts, and in a refinement it re-evaluates some of these prefixes with 16 rollouts. More than 80\% of its evaluated prefixes have 8 rollouts, and more than 70\% of its \pit{} steps were evaluated with 16 (Table~\ref{tab:collections}). About half of the 16-rollout prefixes of \pit{}s lie within two evaluated steps of the \pit{} step. The collection evaluates about 60\% of the steps of a trace.

Each set of $K$ rollouts contains one greedy rollout and $K-1$ independent samples at temperature $0.7$, so its success count is not binomial. Writing $G_t\in\{0,1\}$ for the greedy result and $p_t$ for the temperature-$0.7$ success probability, the count is $G_t+X_t$, where $X_t\sim\mathrm{Bin}(K-1,p_t)$. Thus $\mathbb{E}[\hat p_t\mid s_t]=[G_t+(K-1)p_t]/K$, a $K$-dependent protocol mean rather than $p_t$. Rollout rates at $K=8$ and $K=16$ therefore differ in mixture weight as well as sampling variance under the recorded protocol. Conditional on the greedy outcome, the remaining $K-1$ samples give rate standard deviations of $0.099$ at $K{=}8$ and $0.073$ at $K{=}16$ near $p{=}0.1$. A depth sweep reports refinement-depth agreement of $0.754$ at $K{=}16$ against $0.638$ at $K{=}8$. The two agreement values are point estimates against a deeper reference, with no interval.

The envelope is computed over the prefixes evaluated by a collection (Section~\ref{sec:failure}). Rollouts continue from the final recorded prefix, so the value of that prefix is estimated from rollouts and is not set equal to the final label.

The rollout search, replayed offline on the gold collection's stored rollout rates, reads every second step and then bisects the interval that contains the first prefix at which the envelope is at or below $\theta$. No 8B replay exists.

Each rollout is capped at $4{,}096$ tokens, a cap read from the stored rollouts at both scales. The temperature of $0.7$ is taken from our experiment records. The temperature, nucleus-sampling parameter, random seed and stop strings cannot be recovered from the stored rollouts, and the script that generated the rollouts was not retained.

\begin{table}[htbp]
\centering\small
\caption{The two collections of rollouts. A dash marks a quantity that was not computed for the gold collection. The refinement re-evaluates some prefixes with more rollouts than the scan.}
\label{tab:collections}
\begin{tabular}{@{}p{6.4cm}ccc@{}}
\toprule
& \multicolumn{2}{c}{broad collection} & gold collection \\
\cmidrule(lr){2-3}\cmidrule(l){4-4}
& 4B & 8B & 4B \\
\midrule
evaluated steps & \multicolumn{3}{c}{0, 1, 2, then every second step} \\
rollouts per prefix in the scan & 8 & 8 & 4 \\
rollouts per prefix in the refinement & 16 & 16 & 8 \\
evaluated prefixes with the scan's rollouts & $82.3\%$ & $83.3\%$ & $81.4\%$ \\
evaluated prefixes with the refinement's rollouts & $17.7\%$ & $16.7\%$ & $18.6\%$ \\
\pit{} steps evaluated with the refinement's rollouts & $8{,}270$ & $6{,}895$ & 448 \\
\pit{} steps evaluated with the scan's rollouts & $2{,}990$ & $2{,}613$ & 370 \\
mean share of a trace's steps evaluated & $0.620$ & $0.612$ & $0.623$ \\
cap on each rollout, tokens & $4{,}096$ & $4{,}096$ & --- \\
\bottomrule
\end{tabular}
\end{table}

The gold collection contains $5{,}380$ Qwen3-4B traces and supplies 818 \pit{}s, whose \pit{} steps are the reference \pit{} steps. It scans with 4 rollouts and refines with 8 (Table~\ref{tab:collections}), so no step in it receives the 16 rollouts of the broad collection's refinement. At $K{=}4$ and $K{=}8$, a rollout rate at or below $\theta$ means zero successes.

The gold collection overlaps the broad collection. Of its $5{,}380$ traces, $2{,}760$ are also broad-collection traces, including 815 of the 818 gold \pit{} traces. For 84 of these 815 traces, the broad collection assigns a different failure type: 76 capability gaps and 8 near misses.

This overlap restricts the traces that can be evaluated at 4B. For the pointer probe evaluated on the reference \pit{} steps, we exclude every trace from every problem that contributes a gold \pit{}. The value probe is trained on broad-collection traces, which include 815 of the gold \pit{} traces. For a given seed, 528 to 565 of the gold \pit{}s are therefore training traces of the value probe. The final-label probe's training traces and problem split are not recorded. Section~\ref{sec:wall} reports an evaluation of each seed on the problems held out from the value probe's training, with 205 to 242 \pit{}s per seed. Pooled over all reference \pit{} steps, about two thirds of each seed's evaluations are on training traces of the value probe. In this pooled evaluation the value probe scores higher than the final-label probe, and both probes are below the length rule in all 20 runs (two probes, two input sets, five seeds; Table~\ref{tab:paired}).

At 8B, where no gold collection exists, the value probe is scored on the \pit{} steps of the broad collection within each seed's held-out problems, $2{,}805$ to $2{,}918$ per seed. The length rule is refitted on each seed's training problems. The 8B comparison uses the five seed means and reports sample standard deviations over seeds in place of a bootstrap interval on the difference.

\paragraph{Label corrections.}
The stored final labels mark some correct answers as incorrect. The mislabelled answers are yes/no answers given as their true/false synonyms, and answers to true-or-false questions whose reference answer is coded as a letter. The reported numbers use the stored labels, and the probes are trained on them.

Correcting these labels changes the 8B capability-gap miss rate under the problem split, from $0.797$ to $0.779$, which is the largest change among the reported quantities. The 8B \pit{} miss rate moves from $0.815$ to $0.800$. The shares of the failure types move by at most $0.006$, detection AUROC by at most $+0.004$, and the difference between \pit{}s and near misses after controls by at most $+0.005$.

The stored labels also mark as incorrect the answers that the grader could not parse. There are $2{,}853$ such capability-gap and \pit{} traces at 4B and $2{,}994$ at 8B. Removing these unparsed traces leaves capability gaps and \pit{}s missed more often than near misses at both scales.

Removing instead every problem whose reference answer is a yes/no answer or a letter removes 47\% of the incorrect 4B traces. On the remaining traces, the interval of the 4B difference between \pit{}s and near misses after controls on the problem split is $[+0.009,+0.085]$, which contains the minimum effect. On all traces the interval is $[+0.012,+0.064]$, which lies below the minimum effect.

\paragraph{Seeds and intervals.}
The confidence probe, combined probe, value probe, and pointer probe each use five seeds, with new weights for each seed and a new problem split for each seed, except the pointer probe, whose problem split is the same for every seed. The final-label probe uses the same five seeds; its problem split is not recorded. The value probe, the final-label probe, and the pointer probe use the common set $\{42,137,271,314,577\}$. The combined probe also uses the common set; the confidence probe uses seeds $0$ to $4$. Read-site selection uses bootstrap intervals without a seed axis.

The comparison of the confidence probe's predicted correctness at the \pit{} step with other steps uses, for each trace, only the seeds that held it out. The failure-type classifier uses six repetitions of five-fold cross-validation grouped by problem, with the same folds for the classifier and every baseline; the margin is applied to the mean difference over the six repetitions, and each repetition has its own bootstrap interval (Appendix~\ref{app:presence}). Every bootstrap resamples problems and uses between $1{,}000$ and $2{,}000$ replicates, except the read-site intervals, which use $1{,}000$ resamples of held-out traces, and the pooled differences of the final-label comparison in Table~\ref{tab:support}, which use $10{,}000$ resamples of traces.

\paragraph{Residualization.}
For Section~\ref{sec:blind}, we fit a logistic model of trace misses on failure type, trace length, answer position, difficulty, and best reachable value at step 0, $V^*_0$. The difference between \pit{}s and near misses after controls is the average marginal effect obtained by changing the failure type from near miss to \pit{} while holding each other variable at its observed value, then averaging over traces. For the detection AUROC after controls, we regress trace length, difficulty, problem source, answer length, and numeric-answer status out of the confidence probe's log odds. We score the residual on held-out traces and average over five seeds. As a stratified check, we also compare the miss rates of \pit{}s and near misses within cells of $V^*_0$ bin by trace-length tertile that hold at least 30 near misses. The \pit{} miss rate minus the near-miss miss rate is positive in all five such cells at 4B and in seven of eight at 8B, where the exception holds 33 near misses.

\begin{table}[htbp]
\centering\small
\setlength{\tabcolsep}{4pt}
\caption{Training configuration of every probe and of the failure-type classifier. The four probes read hidden states at the read site (Appendix~\ref{app:repr}); the failure-type classifier reads transcoder activations at layers 19 to 24 (Appendix~\ref{app:presence}). The pointer probe was trained at 4B only.}
\label{tab:probes}
\begin{tabular}{@{}>{\raggedright\arraybackslash}p{2.2cm}>{\raggedright\arraybackslash}p{3.0cm}>{\raggedright\arraybackslash}p{3.1cm}>{\raggedright\arraybackslash}p{2.0cm}l@{}}
\toprule
model & target & training traces & mask & scales \\
\midrule
confidence probe & final label $Y$, broadcast to every step & all $144{,}000$ traces & causal & 4B, 8B \\
value probe & rollout rate $\hat p_t$ at each step & recoverable traces & bidirectional & 4B, 8B \\
combined probe & $Y$ on correct traces, $\hat p_t$ on incorrect ones & $143{,}979$ traces (21 incorrect traces without rollout rates dropped) & bidirectional & 4B, 8B \\
pointer probe & \pit{} step; further outputs, including the failure type & $22{,}877$ incorrect traces of the broad collection, of which $8{,}971$ \pit{}s & bidirectional & 4B \\
failure-type classifier & failure type of the problem, capability gap or \pit{} & one scored trace per problem & --- & 4B, 8B \\
\bottomrule
\end{tabular}
\end{table}

\paragraph{Confidence probe.}
The confidence probe is a causal transformer encoder over the steps of a trace, with a learned positional embedding and a two-layer sigmoid head. The two scales use the same architecture and training settings and differ only in input width and number of parameters (Table~\ref{tab:confprobe}). For each step, the input concatenates the hidden states of the two read-site layers (Appendix~\ref{app:repr}) at one token position.

The objective broadcasts the trace's final label $Y$ to every valid step and minimizes per-step Bernoulli loss, averaged over the valid steps in the batch:
\begin{equation}
\mathcal{L}_{\text{conf}} = -\frac{1}{\sum_i T_i}\sum_i\sum_{t=1}^{T_i} \big[\, Y_i \log c_{i,t} + (1-Y_i)\log(1-c_{i,t})\,\big],
\qquad c_{i,t} = \sigma(g(h_{i,1:t})).
\end{equation}
Here $h_{i,1:t}$ is the sequence of hidden states read at the read site for steps 1 to $t$ of trace $i$, $g$ is the encoder's scalar logit output, and $\sigma$ is the logistic function. A Brier-loss robustness variant remains within the BCE variant's seed spread on every detection and calibration measurement (Table~\ref{tab:conf}). The experiment establishes no difference between the two objectives, which were not compared pairwise.

Each model trains on all traces with the problem split, which holds out about 30\% of the traces, so no problem contributes traces to both training and validation. We fit one temperature on held-out logits. Expected calibration error uses ten equal-width bins and is reported before temperature scaling, which reduces it (Table~\ref{tab:conf}).

As a leakage comparison, we also train with a 70/30 split by trace, stratified to preserve the base correctness rate, which we call the \emph{trace split}. Under that split, $97.6\%$ of validation traces share a problem with at least one training trace. Table~\ref{tab:tracesplit} gives the trace-split results; the 4B difference between \pit{}s and near misses after controls is $+0.085$ under the trace split and $+0.037$ under the problem split. The trace split also gives 4B \pit{}s a higher miss rate than capability gaps ($0.742$ against $0.717$, Table~\ref{tab:tracesplit}); under the problem split the \pit{} miss rate minus the capability-gap miss rate is $-0.003\,[-0.018,+0.012]$.

\begin{table}[htbp]
\centering\small
\caption{Architecture and training settings of the confidence probe. A value spanning both columns applies at both scales.}
\label{tab:confprobe}
\begin{tabular}{@{}p{6.0cm}cc@{}}
\toprule
& 4B & 8B \\
\midrule
pre-norm encoder layers & \multicolumn{2}{c}{4} \\
width & \multicolumn{2}{c}{512} \\
attention heads & \multicolumn{2}{c}{8} \\
maximum number of steps & \multicolumn{2}{c}{48} \\
input dimension per step & $5{,}120$ & $8{,}192$ \\
parameters & 15.5M & 17.1M \\
optimizer and learning rate & \multicolumn{2}{c}{AdamW, $10^{-4}$} \\
batch size & \multicolumn{2}{c}{64} \\
maximum epochs, with early stopping & \multicolumn{2}{c}{20} \\
numeric precision & \multicolumn{2}{c}{\texttt{bf16}} \\
\bottomrule
\end{tabular}
\end{table}

\begin{table}[htbp]
\centering\small
\caption{The confidence probe under the trace split, the leakage comparison. Table~\ref{tab:conf} and Table~\ref{tab:conceal} give the same quantities under the problem split.}
\label{tab:tracesplit}
\begin{tabular}{@{}p{6.0cm}cc@{}}
\toprule
& 4B & 8B \\
\midrule
detection AUROC & $0.8165$ & $0.8165$ \\
AUROC after controls & $0.7322$ & $0.7267$ \\
ECE & $0.027$ & $0.019$ \\
miss rate, capability gap & $0.717$ & $0.795$ \\
miss rate, \pit{} & $0.742$ & $0.823$ \\
miss rate, near miss & $0.469$ & $0.522$ \\
\pit{}s $-$ near misses, after controls & \mbox{$+0.085\,[+0.061,+0.110]$} & \mbox{$+0.119\,[+0.099,+0.137]$} \\
\bottomrule
\end{tabular}
\end{table}

\paragraph{Text baseline.}
The text baseline is a gradient-boosted model on eleven features of the trace text: the densities of sentence roles such as verification, derivation and backtracking, assigned by cue-word lists, statistics of the transitions between roles, the position of the final answer, and the number and mean length of sentences. On the problem split, a logistic regression on the same features reaches a detection AUROC of $0.7548$ at 4B and $0.7379$ at 8B. Adding the confidence probe's score to the text baseline's features raises the detection AUROC by $+0.034$ to $+0.041$ across the two scales and the two training losses, with intervals that exclude zero.

\paragraph{Combined probe.}
Table~\ref{tab:probes} sets out the target, training traces, mask and model scales of every probe. The combined probe, compared with the confidence probe in Figure~\ref{fig:conceal}C, has the confidence probe's architecture, input and optimizer (Table~\ref{tab:confprobe}). Its attention mask is bidirectional. It trains on correct and incorrect traces, taking the final label as the target at every step of a correct trace and the rollout rate $\hat p_t$ as the target at each evaluated prefix of an incorrect trace; the other steps of an incorrect trace carry no loss. It weights each evaluated prefix by its rollout count and each step of a correct trace by one, and the confidence probe weights every step equally. It trains for a fixed 8 epochs, and the confidence probe trains for at most 20 epochs with early stopping. The combined probe has a lower detection AUROC than the confidence probe on held-out problems at both scales (Table~\ref{tab:support}). Training the confidence probe's design on the combined probe's labels may help isolate the contribution of the labels; we did not run this experiment.

\subsection{Failure-type classifier}
\label{app:presence}

\paragraph{Input and population.} The failure-type classifier reads the transcoder activations at the input of the MLP block of each of layers 19 to 24, from the publicly released transcoders for both models\footnote{\texttt{mwhanna/qwen3-4b-transcoders} and \texttt{mwhanna/qwen3-8b-transcoders} on Hugging Face.}. For each problem it takes one completed incorrect trace, the one with the lowest trace index, which we call the scored trace.

Each latent's rectified activation is summed over the trace's tokens, at most 16{,}384, and binarized at that latent's 90th percentile over the problems of the training folds. The binarized vector is projected to 256 dimensions by truncated singular value decomposition and passed to a class-balanced logistic regression. \citet{gallifant2025saefeat} sum each sparse-autoencoder latent over the tokens of a text and binarize the sum at a fixed threshold before a logistic regression. The failure-type classifier uses the same summation and binarization with a threshold set for each latent.

The label of a problem is its failure type in the broad collection, capability gap or \pit{}. At 4B the population is a class-balanced sample of 1{,}722 problems within the difficulty range 7.0 to 9.5: 839 capability gaps and 883 \pit{}s. At 8B it is all 6{,}372 such problems: 2{,}749 capability gaps and 3{,}623 \pit{}s.

The scored trace is not always the trace that set the problem's label. At 4B it differs from that trace for 971 of the 1{,}722 problems. The scored trace is not a verified failure for 13.6\% of problems at 4B and 20.3\% at 8B, because of the grading defects described under Label corrections.

\paragraph{Baselines and evaluation.} The baselines use no hidden states. They are built from six feature sets: the question text with problem difficulty, trace counts and lengths, the question text with difficulty and the counts and lengths combined, the text of the scored trace, five statistics of the model's own token log-probabilities on that trace, and the first five sets together, each fitted by logistic regression and by gradient-boosted trees. The strongest of these twelve baselines is taken in each repetition. Every model uses the same five folds grouped by problem, and the five-fold cross-validation is repeated with six seeds; the margin is applied to the mean over the six repetitions of the difference from the strongest baseline, and each repetition has a problem-clustered bootstrap interval. The same pipeline applied to the rectified hidden states at layers 19 to 24, summed over tokens, gives the comparison with hidden states; the transcoder input has $983{,}040$ dimensions against $15{,}360$ (4B) and $24{,}576$ (8B) for the hidden states, and no random encoding of equal width was tested. Reordering the rows of identical inputs moves a baseline's AUROC by about $0.01$. The text baselines and the comparison with hidden states address two results on probes of sparse features. Between 45 and 90 percent of the sparse-autoencoder features that activate more on reasoning text than on other text also activate when a few associated tokens are inserted into non-reasoning text \citep{ma2026saereasonfeat}. Probes on sparse-autoencoder features do not improve on logistic regression on the dense activations, on average over 113 classification datasets \citep{kantamneni2025sparseprobing}. A linear probe on the hidden state of the last prompt token separates competition problems that a reasoning model fails under greedy decoding from those it solves, before the trace begins \citep{stopbeforefail2026}. A problem that one greedy run fails can be a capability gap, a near miss or a \pit{} under the failure types defined here.

\paragraph{Results and label sensitivity.} With the stored labels, the failure-type classifier meets the margin at 8B and not at 4B (Table~\ref{tab:presence}). The difference from the strongest baseline is $+0.055$ at 8B, with every repetition at or above the margin. It is $+0.049$ at 4B, with three of six repetitions at or above the margin. Every interval excludes zero and contains the margin (Table~\ref{tab:presence-rep}). The same classifier on hidden states at the same layers scores lower at both scales. No probe in Table~\ref{tab:probes} is trained on this label.

Removing problems whose scored trace may carry a wrong label raises the difference at both scales. The removal covers every problem whose scored trace carries a grading defect and, at 4B, every problem whose scored trace has a different failure type from its label. The removal leaves 1{,}294 problems at 4B and 5{,}078 at 8B and raises the difference to $+0.065$ and $+0.083$. With these problems removed, every repetition is at or above the margin at both scales.

The removed problems are 62\% (4B) and 65\% (8B) capability gaps, and the baselines lose more AUROC on the remaining problems than the classifier does. The rise in the difference is consistent with two explanations that this analysis does not separate: grading defects that the question text identifies, and a subset of capability gaps that the question text identifies.

The transcoders reconstruct the MLP outputs at layers 19 to 24 with a fraction of variance explained, pooled over tokens, of 0.338 to 0.525 at 4B and 0.366 to 0.564 at 8B.

\begin{table}[htbp]
\centering\small
\setlength{\tabcolsep}{4pt}
\caption{Failure-type classifier on one completed incorrect trace per problem, capability-gap problems against \pit{} problems. AUROCs are means over six repetitions of five-fold cross-validation; the difference is the classifier's AUROC minus that of the strongest baseline that uses no hidden states, mean $\pm$ standard deviation over repetitions. The hidden-state column is the same classifier on hidden states at the same layers. The last two rows remove the problems whose scored trace carries a grading defect and, at 4B, those whose scored trace has a different failure type from its label.}
\label{tab:presence}
\setlength{\tabcolsep}{3pt}
\begin{tabular}{llcccccc}
\toprule
labels & scale & problems & classifier & \shortstack{hidden\\states} & \shortstack{strongest\\baseline} & difference & \shortstack{repetitions\\$\geq 0.05$} \\
\midrule
stored & 4B & 1{,}722 & 0.667 & 0.580 & 0.618 & $+0.0488 \pm 0.0071$ & 3 of 6 \\
stored & 8B & 6{,}372 & 0.721 & 0.646 & 0.666 & $+0.0549 \pm 0.0026$ & 6 of 6 \\
defects removed & 4B & 1{,}294 & 0.643 & 0.581 & 0.578 & $+0.0648 \pm 0.0074$ & 6 of 6 \\
defects removed & 8B & 5{,}078 & 0.700 & 0.638 & 0.617 & $+0.0830 \pm 0.0049$ & 6 of 6 \\
\bottomrule
\end{tabular}
\end{table}

\begin{table}[htbp]
\centering\small
\setlength{\tabcolsep}{5pt}
\caption{Difference between the failure-type classifier and the strongest baseline that uses no hidden states, per repetition of five-fold cross-validation, with stored labels. Each cell is the AUROC difference with its 95\% problem-clustered bootstrap interval.}
\label{tab:presence-rep}
\begin{tabular}{lcc}
\toprule
repetition & 4B & 8B \\
\midrule
1 & $+0.0364$ [$+0.0040$, $+0.0690$] & $+0.0581$ [$+0.0468$, $+0.0698$] \\
2 & $+0.0558$ [$+0.0278$, $+0.0845$] & $+0.0550$ [$+0.0426$, $+0.0667$] \\
3 & $+0.0505$ [$+0.0218$, $+0.0814$] & $+0.0537$ [$+0.0413$, $+0.0671$] \\
4 & $+0.0491$ [$+0.0189$, $+0.0794$] & $+0.0553$ [$+0.0425$, $+0.0675$] \\
5 & $+0.0463$ [$+0.0158$, $+0.0797$] & $+0.0506$ [$+0.0377$, $+0.0637$] \\
6 & $+0.0547$ [$+0.0222$, $+0.0879$] & $+0.0568$ [$+0.0442$, $+0.0686$] \\
\bottomrule
\end{tabular}
\end{table}

\subsection{Localization}

\paragraph{Localization baselines.}
The length rule predicts $t=\mathrm{round}(f(L-1))$, and we fit the single parameter $f$ on training traces by maximizing Hit@$\pm2$ and score the length rule out of sample; Hit@$\pm2$ is also the evaluation metric of the probes. The median of $t/(L-1)$ at the \pit{} step $t$, the length rule's scale, is $0.75$ at 4B and $0.79$ at 8B. Figure~\ref{fig:wall} shows two fits on the reference \pit{} steps, one by five-fold cross-validation within the gold collection and one on broad-collection \pit{}s outside the gold problems; the Hit@$\pm2$ of the two fits differs by $0.005$. The uniform-step baseline, a step drawn uniformly from the trace, is computed in closed form for each evaluation set.

\paragraph{Pointer probe.}
The pointer probe predicts the \pit{} step of a \pit{} trace directly. It is a four-layer transformer encoder with bidirectional attention over the steps; at each step it reads that step's hidden states at the read site together with those of the final step, and it assigns a softmax score to every step. Training minimizes the cross-entropy to the labelled \pit{} step smoothed by a Gaussian of width one step, and inference selects the most probable step. The encoder has three further outputs trained jointly, a three-way failure type, a per-step estimate of the envelope and a confidence output, with loss weights 1, 0.25 and 0.5 against 1 for the step distribution; only the step distribution is evaluated in this paper. All 22,877 incorrect training traces (8,971 \pit{}s, 12,692 capability gaps and 1,214 near misses) supervise the failure-type and envelope outputs, and only the \pit{} traces supervise the step distribution. The two weightings of the failure-type loss are inverse-frequency weights (0.85 for \pit{}s, 0.60 for capability gaps, 6.28 for near misses) and uniform weights; the step loss is the same under both. A held-out \pit{} trace counts as a hit when the most probable step lies within two steps of the labelled \pit{} step; held-out traces of the other failure types are not scored for localization. The run that excludes the gold problems trains on 17,143 traces.

\paragraph{Value probe.}
The value probe uses the combined probe's design, trains only on recoverable traces, and predicts the raw rollout rate $\hat p_t$. A second input set appends to each step's hidden states five statistics of the model's next-token distribution at the step's read position: the entropy, the varentropy, the probability of the most probable token, the difference between the probabilities of the two most probable tokens, and the probability mass of the five most probable tokens. With rollout count $K_t$ at step $t$, the value probe minimizes count-weighted Bernoulli cross entropy:
\begin{equation}
\mathcal{L}_{\text{val}}
= \frac{\sum_t K_t \big[\, -\hat p_t \log \widehat V_t - (1 - \hat p_t)\log(1 - \widehat V_t)\,\big]}{\sum_t K_t}.
\end{equation}
A step estimated with $K{=}16$ therefore receives twice the weight of a step estimated with $K{=}8$. The value probe's predicted \pit{} step applies Definition~\ref{def:pit} to its predicted values at the evaluated prefixes: it is the first evaluated step at which the maximum of the predictions at that and every later evaluated prefix is at or below $\theta$, and a \pit{} whose predictions never meet this condition counts as a miss. The final-label probe is a transformer encoder over the steps with a causal attention mask, a value output and a confidence output, trained on final labels; it predicts the step with the largest drop in its value output $q_t$ from the preceding step, $\arg\min_t\,(q_t - q_{t-1})$. Prior work trains a verifier on final labels to estimate, without sampling, each prefix's probability of reaching a correct answer \citep{lu2024autopsv}.

\paragraph{Distance score.}
The distance score in Figure~\ref{fig:wall} uses the 76 \pit{} traces, from 74 problems, that have at least one correct trace of the same problem, at layers 8 to 36, under four distance measures; it was computed from hidden states of an earlier extraction whose read positions were not checked. The distance score in Table~\ref{tab:support} is computed at the read site, from the stored hidden states that the probes read, on every \pit{} trace with at least one correct trace of the same problem. Its AUROC pools the steps of these traces, and $t/(L-1)$ for step $t$ of a trace with $L$ steps is scored in the same way.

\paragraph{Decodability and behavior.}
In a study of in-context learning on three model families at 7B to 9B parameters, function vectors steered the output in 5 of 72 task and model cells in which the logit lens, a tuned lens and a two-layer probe all failed to decode the answer \citep{nadaf2026steerable}. A representation can therefore affect behavior without being decodable by a tested probe.

\section{Proofs}
\label{app:proofs}

\paragraph{Preliminaries.}
A prefix is a finite sequence of tokens, so $s_t$ takes values in a countable set and every conditional expectation below conditions on an event of the form $\{s_t=s\}$; $h_{1:t}$ is a deterministic function of $s_t$, so the same holds for it.

Strict propriety pins the population minimizer of the loss in Proposition~\ref{prop:detect} pointwise, so that minimizer is $\mathbb{E}[Y\mid h_{1:t}]$ under the law the training set draws from. Reading it as the step's on-policy value needs one further condition, representation sufficiency: $\mathbb{E}[Y\mid s_t]$ has to be a function of $h_{1:t}$, so that the read site keeps whatever the prefix carries about the outcome. This paper does not test that condition, and a last-token read at two fixed layers can discard prefix information that predicts the outcome. Calibration is checked empirically (Table~\ref{tab:conf}); subject to sufficiency, one final label per trace suffices to supervise the detection task in Section~\ref{sec:blind}.

Both directions of Proposition~\ref{thm:nonident}'s equivalence follow from the $L^2$ projection of $V^*_t$ onto $\sigma(s_t)$, the $\sigma$-algebra generated by $s_t$, and the projection is the best square-integrable function of $s_t$, the conditional mean $\mathbb{E}[V^*_t\mid s_t]$; a function of $h_{1:t}$ does no better, since $\sigma(h_{1:t})\subseteq\sigma(s_t)$.

Figure~\ref{fig:anat} shows that $V^*$ immediately before the \pit{} step is low and right-skewed, and its excess over $V^\pi_t$, $m_t$, need not be fixed by the current prefix. The threshold event also depends on the greedy outcome, as Lemma~\ref{lem:posthoc} states.

\begin{proof}[Proof of Proposition~\ref{prop:detect}]
Throughout, $\mathbb{Q}$ is the step-weighted law named in the statement: the training set contains every valid step of every trace, so a trace with $n$ valid steps contributes $n$ rows, and $\mathbb{Q}$ is the trace law reweighted by that row count. The weighting changes the marginal over prefixes and nothing else. A prefix $s_t$ carries its own step index, so each trace whose prefix at that index is $s_t$ contributes exactly one row there; conditioning on $s_t$ therefore returns the trace law $\pi$ given that prefix, and $\mathbb{E}_{\mathbb{Q}}[Y\mid s_t]=\mathbb{E}_{\pi}[Y\mid s_t]=V^\pi_t$. Conditioning on $h_{1:t}$ instead gives the corresponding hidden-state conditional mean under the step-weighted law.

Fix a value $h$ of the hidden prefix and write $q_h=\mathbb{E}_{\mathbb{Q}}[Y\mid h_{1:t}=h]$. Conditional on $h_{1:t}=h$ the label is Bernoulli with mean $q_h$, so the population risk of a candidate $g$ under a scoring rule $\ell$ is
\begin{equation*}
R(g)=\mathbb{E}_{\mathbb{Q}}\big[\,\ell(g(h_{1:t}),Y)\,\big]
     =\mathbb{E}_{h\sim\mathbb{Q}}\Big[\;\mathbb{E}_{Y\sim\mathrm{Bern}(q_h)}\big[\ell(g(h),Y)\big]\;\Big].
\end{equation*}
The minimization runs over measurable functions, which place no constraint tying the value of $g$ at one $h$ to its value at another, so $R$ is minimized by minimizing the inner expectation separately at each $h$.

For the log loss the excess conditional risk at $h$ is
\begin{equation*}
-q_h\log\frac{g(h)}{q_h}-(1-q_h)\log\frac{1-g(h)}{1-q_h}
  =D_{\mathrm{KL}}\big(\mathrm{Bern}(q_h)\,\big\|\,\mathrm{Bern}(g(h))\big)\;\ge\;0 ,
\end{equation*}
which is zero only at $g(h)=q_h$. The Brier rule is strictly proper as well and has the same minimizer, its excess conditional risk being $(g(h)-q_h)^2$. Both are instances of a strictly proper scoring rule pinning the conditional minimizer at the conditional mean \citep{gneiting2007proper}. The conditional minimizer equals $q_h$ at every $h$, so the risk minimizer is $g^\star(h_{1:t})=\mathbb{E}_{\mathbb Q}[Y\mid h_{1:t}]=\Pr_{\mathbb Q}(Y{=}1\mid h_{1:t})$.

The hidden states are a deterministic function of the prefix, so $\sigma(h_{1:t})\subseteq\sigma(s_t)$, and the tower property gives $g^\star(h_{1:t})=\mathbb{E}_{\mathbb Q}\big[\,\mathbb{E}_{\mathbb Q}[Y\mid s_t]\,\big|\,h_{1:t}\big]=\mathbb{E}_{\mathbb Q}[V^\pi_t\mid h_{1:t}]$, so $g^\star(h_{1:t})=V^\pi_t$ almost surely if and only if $\mathbb{E}[Y\mid s_t]$ is $\sigma(h_{1:t})$-measurable up to null sets, which is the representation-sufficiency condition.
\end{proof}

\begin{proof}[Proof of Proposition~\ref{thm:nonident}]
Write $p=V^\pi_t$, which is a function of $s_t$, and recall $m_t=V^*_t-V^\pi_t\ge 0$.

The projection identity below holds for every $\sigma(s_t)$-measurable $f\in L^2$, since $V^*_t\in[0,1]$; the projection of $V^*_t$ onto $\sigma(s_t)$ splits the squared error into the part no such $f$ can remove and the part that measures the distance of $f$ from the conditional mean \citep[Ch.~9]{williams1991martingales}:
\begin{equation*}
\mathbb{E}\big[(V^*_t-f(s_t))^2\big]
 =\underbrace{\mathbb{E}\big[\mathrm{Var}(V^*_t\mid s_t)\big]}_{\text{irreducible}}
 +\underbrace{\mathbb{E}\big[(\mathbb{E}[V^*_t\mid s_t]-f(s_t))^2\big]}_{\text{choice of }f} .
\end{equation*}
If $\mathrm{Var}(V^*_t\mid s_t)=0$ almost surely, take $f=\mathbb{E}[V^*_t\mid s_t]$: both terms vanish, the left side is zero, and so $f(s_t)=V^*_t$ almost surely. Conversely, if some $f$ makes the left side zero, then both non-negative terms vanish; $\mathbb{E}[\mathrm{Var}(V^*_t\mid s_t)]=0$ with a non-negative integrand gives $\mathrm{Var}(V^*_t\mid s_t)=0$ almost surely.

The projection uses only the measurability of $f$ with respect to $\sigma(s_t)$. A training label and loss select one function from that class but do not add information. The squared-loss optimum over functions of the hidden prefix, where $\sigma(h_{1:t})\subseteq\sigma(s_t)$, is the projection $\mathbb{E}[V^*_t\mid h_{1:t}]$, with the same or greater irreducible error.

The same equivalence holds for the discrete $\tpit$: it takes values in $\mathbb N_0\cup\{\infty\}$, so the argument runs through its conditional law. If $\tpit=f(s_t)$ almost surely for a $\sigma(s_t)$-measurable $f$, then conditionally on $s_t$ the variable $\tpit$ equals the constant $f(s_t)$ almost surely, so its conditional law is the point mass at $f(s_t)$. Conversely, if that conditional law is the point mass at some $c(s_t)$, then $\Pr(\tpit=c(s_t)\mid s_t)=1$ almost surely, and taking expectations gives $\tpit=c(s_t)$ almost surely, and countability makes the unique-atom selector $c$ measurable.

Subtracting $V^\pi_t$ does not change conditional variance, because $V^\pi_t$ is $\sigma(s_t)$-measurable: $\mathrm{Var}(m_t\mid s_t)=\mathrm{Var}(V^*_t\mid s_t)$. The condition $0<V^\pi_t<1$ by itself does not imply that either variance is positive. For example, suppose every future reasoning prefix reachable from $s_t$ has the same value $p\in(0,1)$ until the final answer is sampled. Then a fresh continuation outcome remains random, while $V^*_t=p$ and $m_t=0$ are fixed by $s_t$. Positive conditional variance instead requires continuation paths of positive probability that produce different future maxima. In that case, the first term in the projection identity is the irreducible mean-squared error of every prefix-state predictor. A hidden-state predictor can have additional error because $h_{1:t}$ may discard information from $s_t$. Thus the proposition gives a conditional obstruction to exact recovery. The size of that obstruction depends on $\mathrm{Var}(V^*_t\mid s_t)$; $V^\pi_t$ alone does not determine it.
\end{proof}

\begin{lemma}[The completed path determines the \pit{}]
\label{lem:posthoc}
Fix the inclusive crossing test $V^*_t\le\theta$ of Definition~\ref{def:pit}. \textbf{(i)} Let $V^\pi_0,\dots,V^\pi_T$ be the exact value path over the recorded reasoning prefixes used to compute the envelope for a completed incorrect trace, and let $L=\max\{u\in\{0,\dots,T\}:V^\pi_u>\theta\}$ with $\max\emptyset=-1$. Then $\tpit=0$ (capability gap) if $L=-1$, $\tpit=L+1$ (\pit{}) if $-1<L<T$, and $\tpit=\infty$ (near miss) if $L=T$. \textbf{(ii)} Let $G_t\in\{0,1\}$ be the greedy-continuation outcome and $X_t\sim\mathrm{Bin}(K-1,p_t)$ count successes among the independent temperature-$0.7$ rollouts. Then $\hat p_t=(G_t+X_t)/K$ and $\hat p_t\le\theta\iff X_t\le\lfloor\theta K\rfloor-G_t$.
\end{lemma}

\begin{proof}
(i) If $L=-1$, every value is at most $\theta$, so the envelope crosses at step $0$. If $-1<L<T$, the envelope exceeds $\theta$ through $L$ and is at most $\theta$ afterward, so the first crossing is $L+1$. If $L=T$, every candidate crossing has a later value above $\theta$, so the envelope never crosses and $\tpit=\infty$. The three cases of $L$ correspond one-to-one to the capability-gap, \pit{} and near-miss cases of Definition~\ref{def:pit}. Under the strict test $V^*_t<\theta$, the same cases hold with $L$ taken over $\{u:V^\pi_u\ge\theta\}$.

(ii) The total count is $G_t+X_t$, so the threshold equivalence follows by multiplying by $K$ and using that $X_t$ is integer-valued. At $\theta=0.1$, a $K=8$ crossing requires $G_t=0$ and $X_t=0$. At $K=16$, it requires $X_t\le1$ when $G_t=0$ and $X_t=0$ when $G_t=1$. Thus the mixed rollout design changes both the threshold grid and the sampling law of the success count across depths.
\end{proof}

\section{Read-site selection}
\label{app:repr}

The confidence, value, combined and pointer probes read hidden states at the read site $\mathcal{S}^\star=\{23,24\}$. We select the read site by the detection AUROC of a probe at each layer, using the final label $Y$ only. The selection measure is the detection AUROC in the middle step bin, computed within each difficulty group and then averaged, with a probe of the confidence probe's architecture trained on the hidden states of each layer. The middle bin reduces both the difficulty prior present in early steps and the answer commitment present in the final step. Figure~\ref{fig:layers}A also shows the same measure for transcoder activations, obtained by applying the published per-layer transcoders to the residual stream before the layer normalization that precedes the MLP, with their activations projected to the width of the hidden states by a fixed random matrix. The transcoders were trained on the output of that normalization, so the transcoder curves do not describe the transcoders at their training site. No result in this paper uses the transcoder curves. We also record predictive $\mathcal{V}$-information \citep{xu2020vinformation,ethayarajh2021vusable}; predictive $\mathcal{V}$-information peaks at layer 28 at 4B and at layer 18 at 8B and was not used for the selection. We record a usable-information measure because raw mutual information cannot order deterministic layers \citep{pimentel2020probingmi,xu2020vinformation}.

\paragraph{Layer selection.}
The detection AUROC forms a broad upper-middle-layer plateau and peaks at layer 23 for both model scales, at $0.7264$ $[0.7079, 0.7445]$ for 4B and $0.7124$ $[0.6933, 0.7295]$ for 8B (Figure~\ref{fig:layers}A); the interval at the peak overlaps those of most neighbouring layers. The read site $\{23,24\}$ is not distinguishable from layer 23 alone because the four bootstrap intervals, the two at layer 23 above and the two read-site intervals in Table~\ref{tab:support}, overlap. This selection uses one measurement per scale with no seed spread, and the two scales, selected separately, peak at the same layer. The upper-middle location is consistent with prior truth and correctness probes \citep{azaria2023saplma,marks2023geometryoftruth,burns2022ccs}. Prior work also reports that the best probe layer tends to be deeper in larger models \citep{oneill2025singledirection}, so we repeat the selection at each scale. Conditional probing measures only the information added beyond a baseline such as the first layer \citep{hewitt2021conditionalprobing}, and adjacent residual-stream layers are expected to be redundant. The layer-count sweep tests 2, 3, 5 and 8 adjacent layers centred on the peak. Two layers is the smallest tested layer count whose interval overlaps the interval of the best-scoring layer count, and adding a third changes the score by less than the interval width at both scales (Figure~\ref{fig:layers}B).

\begin{figure}[htbp]
\centering
\includegraphics[width=\textwidth]{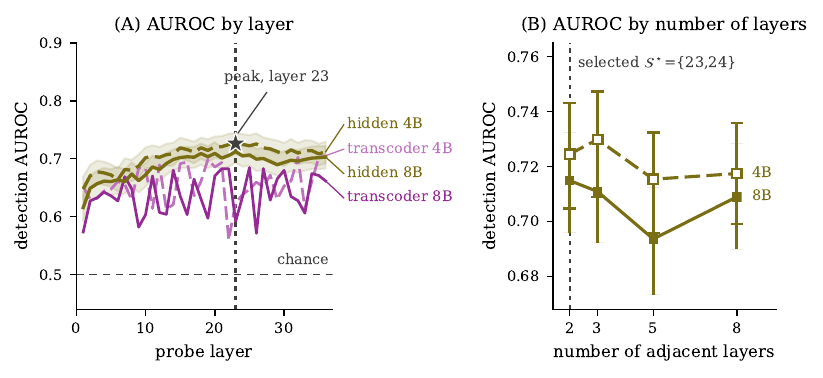}
\caption{Read-site selection for Qwen3-4B and Qwen3-8B. \textbf{(A)} Detection AUROC of a probe at each layer, on hidden states and on transcoder activations (not used by any result in this paper; see the text of this appendix); on hidden states it forms a broad upper-middle-layer plateau and peaks at layer 23. \textbf{(B)} Detection AUROC of a probe on 2, 3, 5 and 8 adjacent layers centred on layer 23; the gain from a third layer over two is smaller than the interval width. Bands and bars are 95\% bootstrap intervals that resample held-out traces. We select hidden states from $\mathcal{S}^\star=\{23,24\}$.}
\label{fig:layers}
\end{figure}

\paragraph{Read position.}
Each read lies at or before the step's last token. The stored step positions come from tokenizing each step on its own, and the hidden states were extracted in left-padded batches without shifting those positions. Each read therefore lies earlier than the step's last token by the padding width of its trace, at most 512 tokens. For the final step, the read lies inside the intended step for $51.0\%$ of traces at 4B and $51.6\%$ at 8B. Tokenizing the steps separately adds a smaller drift, with a median of 3 tokens over 820 step reads in a 40-trace 4B sample. No result was re-scored with corrected positions.

\section{Failure types by problem type}
\label{app:types}

Table~\ref{tab:types} reports the capability-gap share of incorrect traces for each problem type in the broad collection. A problem's type is its DeepMath topic at the second level of that dataset's hierarchy, or its Hendrycks MATH subject. The association between failure type and problem type is weak at both scales, with Cram\'er's $V$ of $0.095$ at 4B and $0.101$ at 8B over traces with the stored labels, and $0.099$ and $0.110$ with corrected labels (Appendix~\ref{app:proto}). The DeepMath topics span a narrow range of capability-gap shares, with Discrete Mathematics above the pooled share and Calculus and Number Theory below it at both scales. The Hendrycks MATH subjects have higher capability-gap shares. The two sources also differ in difficulty scale, so this table does not separate source from topic, and the pooled $V$ carries the same limitation.

\begin{table}[htbp]
\centering\small
\setlength{\tabcolsep}{5pt}
\caption{Capability-gap share of incorrect traces by problem type in the broad collection at both model scales, with the stored labels and $95\%$ problem-clustered bootstrap intervals. The pooled shares are $0.5555$ at 4B and $0.5321$ at 8B. DeepMath topics and Hendrycks MATH subjects are listed separately because the two sources differ in difficulty scale.}
\label{tab:types}
\begin{tabular}{lcccc}
\toprule
problem type & 4B problems & 4B share & 8B problems & 8B share \\
\midrule
\multicolumn{5}{l}{\textit{DeepMath topics}} \\
Algebra & 4{,}136 & 0.563 [0.549, 0.577] & 3{,}728 & 0.537 [0.522, 0.552] \\
Calculus & 2{,}564 & 0.511 [0.494, 0.530] & 2{,}256 & 0.477 [0.457, 0.497] \\
Other & 928 & 0.542 [0.513, 0.572] & 826 & 0.510 [0.478, 0.541] \\
Discrete Mathematics & 890 & 0.603 [0.573, 0.632] & 801 & 0.599 [0.564, 0.627] \\
Geometry & 856 & 0.555 [0.524, 0.584] & 769 & 0.535 [0.506, 0.569] \\
Applied Mathematics & 586 & 0.561 [0.523, 0.596] & 524 & 0.533 [0.491, 0.571] \\
Number Theory & 424 & 0.495 [0.447, 0.538] & 405 & 0.469 [0.424, 0.513] \\
Differential Equations & 239 & 0.512 [0.448, 0.570] & 195 & 0.487 [0.418, 0.549] \\
Precalculus & 178 & 0.483 [0.406, 0.554] & 154 & 0.488 [0.408, 0.566] \\
\midrule
\multicolumn{5}{l}{\textit{Hendrycks MATH subjects}} \\
geometry & 163 & 0.772 [0.722, 0.818] & 159 & 0.789 [0.737, 0.839] \\
intermediate algebra & 84 & 0.566 [0.454, 0.670] & 83 & 0.635 [0.534, 0.733] \\
counting and probability & 78 & 0.810 [0.734, 0.876] & 64 & 0.803 [0.725, 0.872] \\
precalculus & 68 & 0.587 [0.458, 0.693] & 56 & 0.561 [0.443, 0.671] \\
number theory & 60 & 0.641 [0.532, 0.745] & 57 & 0.655 [0.529, 0.757] \\
prealgebra & 38 & 0.768 [0.660, 0.859] & 30 & 0.706 [0.560, 0.844] \\
algebra & 18 & 0.880 [0.756, 0.967] & 16 & 0.682 [0.488, 0.837] \\
\bottomrule
\end{tabular}
\end{table}

\clearpage
\section{Supporting tables}
\label{app:tables}

\begin{table}[htbp]
\centering\small
\setlength{\tabcolsep}{5pt}
\caption{Sizes of the trace sets and of the populations drawn from them, and the use of each population. The broad collection labels every incorrect trace at both model scales. The gold collection is a 4B collection with half as many rollouts per evaluated prefix, and the \pit{} steps of its \pit{}s are the reference \pit{} steps. A later table reports a separate $n$ when its evaluation population differs from the corresponding row here.}
\label{tab:sets}
\begin{tabular}{@{}l>{\raggedright\arraybackslash}p{3.3cm}>{\raggedright\arraybackslash}p{3.2cm}cc@{}}
\toprule
set & population & used for & 4B & 8B \\
\midrule
all traces & every sampled trace (35{,}854 problems) & training the confidence and combined probes & 144{,}000 & 144{,}000 \\
\quad held-out traces & traces whose problems were held out by at least one of the five training seeds & detection and calibration of the confidence probe & 119{,}736 & 119{,}756 \\
broad collection & every incorrect trace & shares of the failure types (Figure~\ref{fig:anat}) & 28{,}699 & 24{,}871 \\
\quad held-out incorrect traces & incorrect traces among the held-out traces & miss rates by failure type (Section~\ref{sec:blind}) & 23{,}837 & 20{,}590 \\
\quad held-out \pit{}s & \pit{}s of problems outside the training split & pointer probe (4B); value probe, per seed (8B) & 2{,}289 & 2{,}805--2{,}918 \\
\quad classifier problems & one incorrect trace per problem, capability gaps and \pit{}s & failure-type classifier (Appendix~\ref{app:presence}) & 1{,}722 & 6{,}372 \\
gold collection & every trace of the collection & reference \pit{} steps & 5{,}380 & \textit{not collected} \\
\quad \pit{}s & reference \pit{} steps & pointer probe trained without gold problems; final-label comparison & 818 & \textit{not collected} \\
\bottomrule
\end{tabular}
\end{table}

\begin{table}[htbp]
\centering\scriptsize
\setlength{\tabcolsep}{4pt}
\caption{Supporting measurements cited in the main text, at both model scales. Intervals are 95\% bootstrap intervals that resample problems unless a note says otherwise, and $\pm$ values are sample standard deviations over seeds (seed axes in Appendix~\ref{app:proto}).}
\label{tab:support}
\begin{tabular}{@{}>{\raggedright\arraybackslash}p{3.9cm}>{\raggedright\arraybackslash}p{3.4cm}>{\raggedright\arraybackslash}p{3.4cm}>{\raggedright\arraybackslash}p{2.5cm}@{}}
\toprule
measurement & 4B & 8B & interpretation \\
\midrule
\multicolumn{4}{@{}l}{\textit{Failure types and detection}} \\
detection AUROC of the read site\textsuperscript{a} & \mbox{0.7245 [0.7047, 0.7432]} & \mbox{0.7149 [0.6960, 0.7326]} & read site chosen from the final label only \\
share of \pit{}s whose \pit{} step equals $\arg\min_t \hat V_t$ & 0.154 & 0.125 & the $\arg\min$ step equals the \pit{} step for a minority of \pit{}s \\
combined probe $-$ confidence probe, detection AUROC on held-out problems & \mbox{$-0.026$ [$-0.029$, $-0.022$]} & \mbox{$-0.025$ [$-0.029$, $-0.021$]} & the confidence probe has the higher detection AUROC \\
confidence probe's predicted correctness at the \pit{} step minus other steps\textsuperscript{b} & \mbox{$-0.0206$ [$-0.0272$, $-0.0140$]} & \mbox{$-0.0351$ [$-0.0372$, $-0.0327$]} & lower at the \pit{} step \\
same difference against the position-matched steps\textsuperscript{b} & \mbox{$-0.0018$ [$-0.0055$, $+0.0022$]} & \mbox{$-0.0051$ [$-0.0062$, $-0.0040$]} & 8B: the decrease remains; 4B: no detected difference \\
\midrule
\multicolumn{4}{@{}l}{\textit{Localization}} \\
final-label comparison, value probe $-$ final-label probe, pooled over all reference \pit{} steps (two input sets)\textsuperscript{c} & \mbox{$+0.0609$ [$+0.0293$, $+0.0929$]}; \mbox{$+0.0433$ [$+0.0137$, $+0.0731$]} & \textit{not evaluated} & pooled evaluation includes training traces of the value probe; on the problems held out from its training the interval excludes zero in 1 of 5 seeds (first input set) and 0 of 5 (second) \\
value probe, AUROC of near misses against \pit{}s by its prediction at the last evaluated step & $0.9148 \pm 0.0054$ & $0.9062 \pm 0.0079$ & separates near misses from \pit{}s \\
value probe, per-step absolute error versus predicting zero and a monotone fit\textsuperscript{e} & $0.1227 \pm 0.0052$ versus $0.1040 \pm 0.0019$ and $0.0830 \pm 0.0014$ & $0.1261 \pm 0.0040$ versus $0.1091 \pm 0.0034$ and $0.0962 \pm 0.0024$ & predicting zero and the monotone fit each have lower error than the value probe in every seed \\
value probe versus length rule, Hit@$\pm2$ on each seed's held-out \pit{}s, two input sets\textsuperscript{d} & Table~\ref{tab:paired} & $0.4293 \pm 0.0223$ and $0.4046 \pm 0.0162$ versus $0.5028 \pm 0.0108$; 0 of 5 seeds each & below the length rule at 8B \\
\pit{}s located within two steps by exactly one of the pointer probe and the length rule, 10 runs & 723 to 857 of 2{,}289 \pit{}s (31.6 to 37.4\%) & \textit{not evaluated} & about a third of \pit{}s in each run \\
distance score at the read site versus $t/(L-1)$, AUROC of the \pit{} step against the other steps\textsuperscript{f} & 0.55 versus 0.63 (gold); 0.55 versus 0.69 (broad) & 0.56 versus 0.74 (broad) & lower than $t/(L-1)$ in each collection and at each scale \\
rollout search replay, reference \pit{} step within two steps at $6.6$ evaluated prefixes per trace & 0.804 & \textit{not evaluated} & reads the estimates that define the labels; not comparable with a probe's Hit@$\pm2$ \\
\quad refinement-depth agreement & 0.754 (0.638 at $K{=}8$) & \textit{not evaluated} & point estimates against a deeper reference; no interval \\
\bottomrule
\end{tabular}
\par\smallskip\begin{minipage}{\linewidth}\footnotesize\textit{Notes.} \textsuperscript{a}Intervals resample held-out traces. \textsuperscript{b}Each trace uses only the seeds that held it out, and \pit{} steps at the first or last step of a trace are excluded; the rows cover 682 reference \pit{} steps at 4B and 7{,}277 \pit{}s of the broad collection at 8B, so the two scales are not a matched comparison. \textsuperscript{c}The final-label comparison is evaluated only at 4B; intervals of the pooled evaluation resample traces. \textsuperscript{d}At 8B, the row compares seed means on 2{,}805 to 2{,}918 held-out \pit{}s of the broad collection per seed and has no per-trace interval. The value-probe rows use the first input set except where two input sets are named. \textsuperscript{e}The monotone fit is the best non-increasing least-squares fit to each trace's rollout rates. \textsuperscript{f}The distance score takes the highest of four distance measures; the comparison scores step $t$ of a trace with $L$ steps by $t/(L-1)$. The row uses the \pit{} traces with at least one correct trace of the same problem: 593 in the gold collection, 7{,}987 in the broad collection at 4B and 6{,}881 at 8B.\end{minipage}
\end{table}

\begin{table}[htbp]
\centering\small
\setlength{\tabcolsep}{5pt}
\caption{Values behind Figure~\ref{fig:conceal}: confidence-probe predictions at the final step, on the broad collection at both model scales, on held-out problems of the problem split. An incorrect trace is missed when its predicted correctness is at least the cutoff, $0.5$. The binned rows give the miss rates of \pit{}s and near misses within bins of the best reachable value at step 0, $V^*_0$, with the number of traces in each bin in parentheses. The AUROC rows score each probe by its final-step score on its own held-out traces. The last block is the difference between \pit{}s and near misses after controls, from a logistic model of trace misses on failure type, trace length, answer position, difficulty and $V^*_0$ (Appendix~\ref{app:proto}). Intervals are 95\% problem-clustered bootstraps.}
\label{tab:conceal}
\begin{tabular}{@{}p{4.3cm}cc@{}}
\toprule
quantity & 4B & 8B \\
\midrule
miss rate, capability gap & \mbox{$0.717\ [0.706, 0.729]$} & \mbox{$0.797\ [0.785, 0.808]$} \\
miss rate, \pit{} & \mbox{$0.714\ [0.703, 0.724]$} & \mbox{$0.815\ [0.805, 0.824]$} \\
miss rate, near miss & \mbox{$0.490\ [0.460, 0.521]$} & \mbox{$0.466\ [0.442, 0.490]$} \\
$n$ (incorrect traces) & 23{,}837 & 20{,}590 \\
\midrule
$V^*_0\le0.25$, \pit{}s & \mbox{$0.707 [0.690, 0.723]$ (3{,}160)} & \mbox{$0.787 [0.770, 0.804]$ (2{,}548)} \\
\quad near misses & \mbox{$0.436 [0.329, 0.547]$ (78)} & \mbox{$0.449 [0.357, 0.549]$ (118)} \\
$0.25<V^*_0\le0.50$, \pit{}s & \mbox{$0.720 [0.701, 0.737]$ (2{,}931)} & \mbox{$0.834 [0.818, 0.849]$ (2{,}415)} \\
\quad near misses & \mbox{$0.424 [0.357, 0.491]$ (205)} & \mbox{$0.468 [0.410, 0.526]$ (299)} \\
$0.50<V^*_0\le0.75$, \pit{}s & \mbox{$0.722 [0.699, 0.745]$ (1{,}591)} & \mbox{$0.822 [0.801, 0.841]$ (1{,}359)} \\
\quad near misses & \mbox{$0.481 [0.410, 0.548]$ (216)} & \mbox{$0.430 [0.371, 0.487]$ (323)} \\
$0.75<V^*_0$, \pit{}s & \mbox{$0.709 [0.687, 0.734]$ (1{,}682)} & \mbox{$0.823 [0.803, 0.843]$ (1{,}568)} \\
\quad near misses & \mbox{$0.517 [0.479, 0.556]$ (722)} & \mbox{$0.479 [0.449, 0.511]$ (1{,}047)} \\
\midrule
AUROC of \pit{}s against near misses, confidence probe & \mbox{$0.670\ [0.653, 0.687]$} & \mbox{$0.760\ [0.746, 0.773]$} \\
\quad $n$ (\pit{}s; near misses) & 9{,}364; 1{,}221 & 7{,}890; 1{,}787 \\
AUROC of \pit{}s against near misses, combined probe & \mbox{$0.364\ [0.349, 0.378]$} & \mbox{$0.403\ [0.390, 0.416]$} \\
\quad $n$ (\pit{}s; near misses) & 9{,}344; 1{,}254 & 7{,}903; 1{,}805 \\
\midrule
\pit{}s $-$ near misses, after controls & $+0.0372\ [+0.0116, +0.0641]$ & $+0.1246\ [+0.1043, +0.1428]$ \\
$n$ (traces in the logistic model)\textsuperscript{a} & 23{,}820 & 20{,}590 \\
\pit{}s $-$ near misses, after controls, Brier loss & $+0.052\ [+0.028, +0.076]$ & $+0.123\ [+0.103, +0.143]$ \\
\bottomrule
\end{tabular}
\par\smallskip\begin{minipage}{\linewidth}\footnotesize\textit{Notes.} \textsuperscript{a}The model omits 17 traces at 4B that contain no final answer, for which answer position is undefined.\end{minipage}
\end{table}

\begin{table}[htbp]
\centering\small
\scriptsize \setlength{\tabcolsep}{2pt}
\caption{Paired comparisons with the length rule, at 4B throughout. Each comparison scores the probe and the length rule on the same \pit{}s; intervals are 95\% problem-clustered bootstraps of the per-trace difference where given.}
\label{tab:paired}
\begin{tabular}{@{}>{\raggedright\arraybackslash}p{3.6cm}ccc>{\raggedright\arraybackslash}p{2.6cm}c@{}}
\toprule
comparison & \pit{}s & probe & length rule & difference & runs above rule \\
\midrule
pointer probe, held-out \pit{}s of the broad collection\textsuperscript{a} & 2{,}289 & $0.4632 \pm 0.0134$ & 0.4692 & $-0.0060$ \mbox{[$-0.0260$, $+0.0141$]} & 0 of 10 \\
pointer probe trained without gold problems & 818 & $0.4293 \pm 0.0120$ & 0.4462 & $-0.0169$ \mbox{[$-0.0473$, $+0.0131$]} & 0 of 5 \\
final-label comparison, held-out: value probe\textsuperscript{b} & 205--242 & 0.3640; 0.3402 & 0.4812 & $-0.166$ to $-0.084$; $-0.182$ to $-0.113$ & 0 of 10 \\
final-label comparison, held-out: final-label probe\textsuperscript{b} & 205--242 & 0.3041; 0.3423 & 0.4812 & $-0.216$ to $-0.124$; $-0.170$ to $-0.116$ & 0 of 10 \\
final-label comparison, pooled, training traces included\textsuperscript{c} & 818 & 0.3186--0.3795 & 0.4462 & $-0.128$ to $-0.067$ & 0 of 20 \\
\bottomrule
\end{tabular}
\par\smallskip\begin{minipage}{\linewidth}\footnotesize\textit{Notes.} All retraining uses seeds $\{42,137,271,314,577\}$. \textsuperscript{a}Ten runs, five seeds under each of two weightings of the failure-type loss; the probe and difference cells report the inverse-frequency weights, and the uniform weights score $0.4524 \pm 0.0111$. The ensemble, the average of the five pointer probes of one weighting, exceeds the length rule by $+0.0262\,[+0.0022,+0.0495]$ with inverse-frequency weights and by $+0.0166\,[-0.0065,+0.0400]$, an interval that contains zero, with uniform weights. \textsuperscript{b}Each seed is scored on the reference \pit{} steps of the problems held out from the value probe's training for that seed; the difference column gives the range of the five per-seed differences, each with an interval below zero. The two values in a cell are the two input sets. \textsuperscript{c}Every seed is scored on all reference \pit{} steps, which include that seed's training traces; every per-seed interval in this row lies below zero.\end{minipage}
\end{table}

\begin{table}[htbp]
\centering\small
\setlength{\tabcolsep}{4pt}
\caption{Detection by the confidence probe at the final step and comparison with the text baseline, on held-out problems of the problem split at both model scales. The upper block reports the mean $\pm$ standard deviation over five seeds. AUROC after controls scores the same five-seed predictions after the five control variables of Appendix~\ref{app:proto} are removed from them. Expected calibration error uses ten equal-width bins on held-out scores, before and after temperature scaling. The lower block gives the held-out detection AUROC of the text baseline and of a logistic regression on the same features, the paired difference between the confidence probe and the text baseline, and the gain from adding the confidence probe's score to the text baseline's features; the lower block's intervals are 95\% problem-clustered bootstraps. Appendix~\ref{app:proto} reports the trace split as a leakage comparison.}
\label{tab:conf}
\begin{tabular}{@{}llcccc@{}}
\toprule
scale & loss & AUROC & AUROC after controls & ECE & ECE after scaling \\
\midrule
4B & BCE (main) & $0.8136 \pm 0.0049$ & $0.7291 \pm 0.0042$ & $0.0143 \pm 0.0040$ & $0.0103$ \\
8B & BCE (main) & $0.8088 \pm 0.0084$ & $0.7221 \pm 0.0149$ & $0.0243 \pm 0.0111$ & $0.0092$ \\
4B & Brier (robustness) & $0.8107 \pm 0.0046$ & $0.7297 \pm 0.0036$ & $0.0273 \pm 0.0103$ & $0.0215$ \\
8B & Brier (robustness) & $0.8085 \pm 0.0047$ & $0.7345 \pm 0.0029$ & $0.0229 \pm 0.0061$ & $0.0107$ \\
\end{tabular}

\medskip
\begin{tabular}{@{}p{5.4cm}cc@{}}
\toprule
& 4B & 8B \\
\midrule
text baseline (gradient-boosted), AUROC & $0.7783$ & $0.7725$ \\
logistic regression on the same features, AUROC & \mbox{$0.7548\ [0.7488, 0.7609]$} & \mbox{$0.7379\ [0.7312, 0.7442]$} \\
confidence probe $-$ text baseline, BCE & \mbox{$+0.0353\ [+0.0319, +0.0387]$} & \mbox{$+0.0362\ [+0.0324, +0.0403]$} \\
confidence probe $-$ text baseline, Brier & \mbox{$+0.0324\ [+0.0290, +0.0357]$} & \mbox{$+0.0359\ [+0.0323, +0.0397]$} \\
gain from adding the probe's score to the text baseline, BCE & \mbox{$+0.0369\ [+0.0335, +0.0402]$} & \mbox{$+0.0406\ [+0.0369, +0.0445]$} \\
gain from adding the probe's score to the text baseline, Brier & \mbox{$+0.0338\ [+0.0305, +0.0370]$} & \mbox{$+0.0393\ [+0.0357, +0.0430]$} \\
\bottomrule
\end{tabular}
\par\smallskip\begin{minipage}{\linewidth}\footnotesize\textit{Notes.} The confidence probe's seed identities are not recorded, so the overlap between its seeds and the value probe's seeds is unknown.\end{minipage}
\end{table}

\end{document}